\documentclass[final]{IEEEtran}

\usepackage[nocompress]{cite}

\usepackage{graphicx}

\usepackage{amsmath,  amssymb, amsthm}
\usepackage{mathtools}
\usepackage[utf8]{inputenc}
\usepackage[T1]{fontenc}

\usepackage{xfrac}

\usepackage{algorithm}
\usepackage{algpseudocode}

\usepackage[pagebackref=true,breaklinks=true,letterpaper=true,colorlinks,bookmarks=false]{hyperref}

\renewcommand{\c}{\mathbf{c}}
\renewcommand{\d}{\mathbf{d}}

\renewcommand{\i}{\mathbf{i}}
\renewcommand{\j}{\mathbf{j}}
\renewcommand{\k}{\mathbf{k}}

\renewcommand{\t}{\mathbf{t}}

\renewcommand{\u}{\mathbf{u}}
\renewcommand{\v}{\mathbf{v}}

\newcommand{\x}{\mathbf{x}}
\newcommand{\y}{\mathbf{y}}
\newcommand{\z}{\mathbf{z}}

\newcommand{\F}{\mathbf{F}}
\newcommand{\G}{\mathbf{G}}

\newcommand{\vOmega}{{\boldsymbol \omega}}
\newcommand{\vTau}{{\boldsymbol \tau}}
\newcommand{\vNu}{{\boldsymbol \nu}}

\newcommand{\Zero}{\mathbf{0}}

\newcommand{\mI}{\mathtt{I}}

\newcommand{\mR}{\mathtt{R}}
\newcommand{\mS}{\mathtt{S}}

\newcommand{\mU}{\mathtt{U}}

\newcommand{\mGamma}{\mathtt{\Gamma}}

\newcommand{\Real}{\mathbb{R}}

\newcommand{\Zahlen}{\mathbb{Z}}

\newcommand{\tC}{\mathcal{C}}

\newcommand{\tG}{\mathcal{G}}

\newcommand{\tQ}{\mathcal{Q}}

\newcommand{\by}{{\times}} 

\newcommand{\sinc}{\mathrm{sinc}}

\newcommand{\norm}[1]{\left\lVert#1\right\rVert} 

\newcommand{\Mat}[2]{\left[\! \begin{array}{#1} #2 \end{array} \!\right]}

\newcommand{\eqdef}{\overset{\underset{\mathrm{def}}{}}{=}}

\newtheorem{theorem}{Theorem}

\newtheorem{proposition}[theorem]{Proposition}

\begin{document}

\title{Multiresolution Search of the Rigid Motion Space \\
  for Intensity Based Registration}

\author{Behrooz Nasihatkon, Fredrik Kahl
  \IEEEcompsocitemizethanks{The authors are  with the Department
of Signals and Systems, Chalmers University of Technology, Sweden, and MedTech West, Sweden. Email: \{behrooz.nasihatkon, fredrik.kahl@chalmers.se\}} }

\maketitle

\begin{abstract}
  We study the relation between the target functions of low-resolution and high-resolution intensity-based registration for the
  class of rigid transformations. Our results show that low resolution target values can tightly bound the high-resolution
  target function in natural images. This can help with analyzing and better understanding the process of multiresolution image
  registration. It also gives a guideline for designing multiresolution algorithms in which the search space in higher
  resolution registration is restricted given the fitness values for lower resolution image pairs. To demonstrate this, we
  incorporate our multiresolution technique into a Lipschitz global optimization framework. We
  show that using the multiresolution scheme can result in large gains in the efficiency of such algorithms.  The method is
  evaluated by applying to 2D and 3D registration problems as well as the detection of reflective symmetry in 2D and 3D images.
\end{abstract}

\section{Introduction}
This paper investigates the implications of low-resolution image registration on the registration in higher resolutions.  We
focus on rigid intensity-based registration with correlation as fitness measure. We show that the fitness computed for lower
resolution image pairs puts a bound on the fitness for higher resolution images. Our results can be exploited for the design and
analysis of multiresolution registration techniques, especially the global optimization approaches.


Most of the approaches to globally optimally solving the registration problem deal with the alignment of point sets
\cite{Huttenlocher1992_Hausdroff, Li2007_3D3Dreg, Yang2013_goicp,Parra2014_3D_reg, Pfeuffer2012_BB_medical} or other geometric
properties \cite{Olsson2009_bb_registration}. For example, Li and Hartley in \cite{Li2007_3D3Dreg} propose an algorithm to
register two sets of points in a globally optimal way based on Lipschitz optimization \cite{Hansen1995_lipschitz}. Yang et al.
\cite{Yang2013_goicp} obtain a globally optimal solution by combing a branch and bound scheme with the Iterative Closest Point algorithm.

Such approaches have not been applied, as much, to intensity-based registration, due to the high computational
cost. Nonetheless, in some applications it is worth to sacrifice efficiency to achieve more accurate results. A notable example
is the construction of Shape Models \cite{Cootes1995_ASM} where the models are trained once and for all. In such a context,
Cremers et al.  \cite{Cremers2008_shape_priors} use Lipschitz optimization to search for optimal rotation and translation
parameters for the alignment of 2D shape priors. After all, the computational burden is still the major limiting factor of such
algorithms. It is well known that the multiresolution techniques can help with speeding up the registration algorithms
\cite{Corvi1995, Thevenaz2000, Maes1999}. However, to use a multiresolution scheme in a global optimization framework we need a
theory which describes the relation among the target values of different resolutions.  We provide such a theory in this paper,
and based on that, propose an algorithm in which many candidate areas of the search space get rejected by examining the target
value in lower resolutions.


This idea of \emph{early elimination} has been well explored in the area of template matching, where a small template patch is
searched in a target image. The successive elimination approach \cite{Li1995_successive} rules out many candidate patch
locations by examining a low-cost lower-bound on the cost function. The approach is extended in \cite{Lee1997_pyramid} to a
multilevel algorithm, each level providing a tighter bound with higher computational cost. A multitude of approaches have been proposed based on the same idea of a
multilevel succession of bounds \cite{Gao2000_multilevel,Chen2001_winner_takes_all, Di2003_bounded_partial_correlation,
  Hel2005_proj_kernels,Mattoccia2008_bounded_corr,Tombari2009_inc_dissimilarity,Ouyang2009_Walsh_Hadamard}.  In
\cite{Ouyang2012_performance_eval} some of the most successful algorithms of this class have been compared.  Among them the work
of Gharavi-Alkhansari \cite{Gharavi2001_low_res_prune} is very related to ours. In this approach each level corresponds to
a different image resolution. By moving from low to high resolution we get tighter bounds requiring more computations. For most
images many of the candidate solutions are eliminated by computing the match measure in lower resolutions. However, in this
method, like all the approaches above,  the search is only performed on a grid of translation vectors with 1 pixel intervals. Most of the
work to extend this to more general transformations do not fully search the parameter space.  We should mention, however, the
work of Korman et al. \cite{Korman2013_affine_template_matching} for affine pattern matching. It considers a grid of affine
transformations and makes use of \emph{sublinear approximations} to reduce the cost of examining each transformation.  It runs a
branch and bound algorithm in which a finer grid is used at each stage. Obtaining a solution close enough to the optimal target
is guaranteed with \emph{high probability}. The bounds, however, are in the form of asymptotic growth rates, and are
\emph{learned} empirically from a large data set. The algorithm, therefore, is not 
provably globally optimal in the exact mathematical sense. 



Our work mainly deals with the alignment of two images rather than matching a small patch. Similar to
\cite{Gharavi2001_low_res_prune} we take a multiresolution approach by providing bounds between different resolutions. However,
since rotation is involved and also sub-pixel accuracy is required, the registration problem is analyzed in the continuous
domain, by interpolating the discrete images. We incorporate the multiresolution
paradigm into a Lipschitz optimization framework which searches the entire rigid motion space for a globally optimal solution.
We show that the multiresolution approach can lead to huge efficiency gains.  Another closely related problem
is estimating reflective symmetry in 2D and 3D images \cite{Kazhdan2004_reflective_symmetry}. We deomostrate  that our algorithm can
readily be applied to this problem as well. In short, the major contributions of this paper are
\begin{itemize}

\item 
  Providing a framework to analyze  the relation among registration targets at different resolutions,

\item presenting inter-resolution bounds for several scenarios such as the use of different interpolation techniques or when only one  image is decimated, and
  giving insights about how tighter bounds could be found,

\item A novel algorithm to integrate the multiresolution paradigm into a Lipschitz global optimization framework resulting in 
  considerable savings in computations.
 
\end{itemize}
We also introduce new effective bounds for the Lipschitz constants which are much smaller than what proposed in
the literature. We should mention, however, that our main concern here is to demonstrate the efficiency gains
of a multiresolution approach, and not optimizing the single-resolution algorithm itself. After providing a background in
Sect. \ref{sec:background}, we introduce the inter-resolution bounds in Sect. \ref{sec:low_res_reg}. A basic grid search
algorithm is presented in Sect. \ref{sec:basic_algorithm}, for better understanding the concept of the multiresolution search,
before presenting the multiresolution Lipschitz algorithm in Sect. \ref{sec:lipschitz_optim}. We conduct experiments in
Sect. \ref{sec:experiments} to evaluate our algorithm on 2D and 3D registration examples, as well as, finding the axis or plane
of symmetry in 2D and 3D images.


\section{Background}
\label{sec:background}
Consider two images $f$ and $g$ represented as continuous functions $\Real^d \rightarrow \Real$, where $d=2$ or $d=3$. For
intensity-based rigid registration we seek to maximize the correlation function
\begin{align}
\label{eq:target}
Q(\mR,\t) = \int f(\x) \, g(\mR (\x + \t))\, d\x,
\end{align}
over $\mR$ and $\t$, the rotation matrix and the translation vector\footnote{Here, all integrals are over $\Real^d$ unless
  otherwise specified.}. This is equivalent to minimizing the integrated squared error
$\int \left(f(\x) - g(\mR (\x {+} \t))\right)^2\, d\x$, as the transformation is rigid.

The fitness function \eqref{eq:target} can be reformulated in the frequency domain. Let $F(\z)$ and $G(\z)$ respectively
represent the Fourier transforms of $f(\x)$ and $g(\x)$, and notice that the Fourier transform of $g(\mR (\x + \t))$ is equal to $G(\mR
\z) \, e^{2\pi i \t^T \z}$. Applying the Parseval's Theorem to \eqref{eq:target} gives
\begin{align}
\label{eq:target_freq}
Q(\mR,\t) = \int \overline{F(\z)} \, G(\mR \z) \, e^{2\pi i \t^T \z}\, d\z.
\end{align}
where $\overline{F(\z)}$ represents the complex conjugate of $F(\z)$.  

\subsection{Discrete images}
For discrete images we only have samples $\{f_{\i}\}$ at discrete locations $\i \in \Zahlen^d$ ($d=2$ or $3$). One could think
of $\{f_{\i}\}$ as samples taken from a continuous image $f$ according to $f_{\i} = f(T\, \i)$, where the scalar $T$ is the
sampling period\footnote{We have assumed that the discrete and continuous coordinates share the same origin.  More
  generally, we can write $f_{\i} = f(T\, \i - \x_0)$, where $\x_0 \in \Real^d$ is the displacement of the origin.}. To get the
continuous image from the discrete image we suppose that the continuous image is bandlimited to the corresponding Nyquist
frequency $\frac{1}{2T}$. This means that the continuous image can be obtained by \emph{sinc} interpolation:
\begin{align}
  \label{eq:interp_sinc}
  f(\x) = \sum_{\i \in \Zahlen^d} f_{\i} ~ \sinc(\frac{\x}{T} - \i),
\end{align}
where $\sinc(x) = \frac{\sin(\pi x)}{\pi x}$ for scalars, and 
\begin{align}
  \label{eq:sinc}
  \sinc(\x) = \prod_{i=1}^d \sinc(x_i),
\end{align}
for $\x \in \Real^d$.  Digital images  are usually defined at a rectangular grid of pixels
$P \subseteq \Zahlen^d$. For consistency with \eqref{eq:interp_sinc}, we extend them by setting $f_{\i} {=} 0$ outside the image
boundaries where $\i {\notin} P$. Thus, the summation in \eqref{eq:interp_sinc} can be over $P$ rather than
$\Zahlen^d$. Notice that the continuous image created by \eqref{eq:interp_sinc} can be nonzero outside the image boundaries at
non-integer coordinates.

The following proposition is useful for numerical integration in the frequency domain, and can be directly verified from the
definition of the Discrete Fourier Transform (DFT).
\begin{proposition}
  \label{thm:DFT_samples}
  Consider a discrete image $\{f_{\i}\}_{\i \in P}$ defined over a \emph{square} grid of pixels $P$ with $n$  pixels along each
  dimension. If the continuous image $f$ is obtained using \eqref{eq:interp_sinc}, the corresponding DFT values are equal to the
  samples of $F(\z)$, the Fourier transform of $f(\x)$, taken at $\frac{1}{nT}$ intervals\footnote{The equality might be
    up to a known global scaling factor depending on the convention used for defining DFT.}.
\end{proposition}

\subsubsection*{Other interpolation techniques}
To accurately perform the sinc interpolation one  needs to consider a large neighbourhood of pixels at any certain
point. Therefore, in practice, kernels with a bounded support are used instead of the sinc function. A large class of
interpolation techniques can be formulated as:
\begin{align}
  \label{eq:interp}
  f'(\x) = \sum_{\i \in \Zahlen^d} f_{\i} ~ s(\frac{\x}{T} - \i),
\end{align}
where $s\colon \Real^d \rightarrow \Real$ is the interpolation kernel, which is typically a low-pass filter. To have a
consistent interpolation, $s(\x)$ needs to be equal to 1 at the origin, and equal to zero at all other discrete sites with
integer coordinates. Two examples of a bounded support kernel are:

\paragraph*{Nearest neighbor interpolation}
In this technique each point in the continuous domain gets the value of the closest discrete pixel. 
This can be obtained by using the box kernel
\begin{align}
  \label{eq:nn_kernel}
  s_{\mathrm{n}}(\x) = 
  \begin{dcases*}
    1 & $\norm{\x}_\infty \leq \frac{1}{2}$, \\
    0 & elsewhere.
  \end{dcases*}
\end{align}

\paragraph*{Bilinear interpolation} In bilinear interpolation (trilinear for 3D) $s$ is the product of triangular kernels in
each dimension. The kernel has a bounded support, and can also be represented as the convolution of a box kernel with itself:
\begin{align}
  s_{\mathrm{t}}(\x) = s_{\mathrm{n}}(\x) * s_{\mathrm{n}}(\x). \label{eq:bilinear_kernel}
\end{align}


\subsection{Computation of the target function}
Assume that $\{f_{\i}\}$ and $\{g_{\i}\}$ are defined on a grid of pixels $P$. By substituting $f'$ (and $g'$) in \eqref{eq:interp} for $f$ and $g$ in
\eqref{eq:target}, we can write the target function corresponding to the discrete image pair $\{f_{\i}\}$ and $\{g_{\i}\}$ as
follows
\begin{align}
  Q(\mR,\t)  &=  \sum_{\i \in P} \sum_{\j \in P} f_{\i} \, g_{\j} \int  s(\frac{\x}{T} - \i)  \,  s(\frac{\mR (\x + \t)}{T} - \j) \, d\x,
  \label{eq:target_discrete0}
  \\
  &=  T^d \sum_{\i \in P} \sum_{\j \in P} f_{\i} \, g_{\j} ~ W_\mR(\mR (\i + \frac{\t}{T}) - \j)
  \label{eq:target_discrete}
\end{align}
where 
\begin{align}
\label{eq:Wd}
W_\mR(\d) =  \int  s(\mR \x + \d)  \,  s(\x) \, d\x.
\end{align}
Notice that $\frac{\t}{T}$ in \eqref{eq:target_discrete} is the translation in the pixel coordinates. Now look at $\mR (\i {+}
\frac{\t}{T}) - \j$.  This is equivalent to rigidly transforming $\i$ according to $\mR$ and $\frac{\t}{T}$, and then taking the
difference with the pixel position $\j$. The weight given to each pair of pixel values $f_\i$ and $g_\j$ is equal to $W_\mR(\mR
(\i {+} \frac{\t}{T}) - \j)$.  The value of $W_\mR(\d)$ is expected to decay as the vector $\d$ grows in size.  If $s$ has a
bounded support like \eqref{eq:nn_kernel} or \eqref{eq:bilinear_kernel}, then $W_\mR(\d)$ will have a bounded support too.
In other cases $W_\mR(\d)$ is negligible for large enough vectors $\d$.  Therefore, for each $\i$, we can sum over pixels $\j$
within a certain neighbourhood of $\i$. This means that computing the target \eqref{eq:target_discrete} needs $O(|P|)$
rather than $O(|P|^2)$ computation, where $|P|$ is the number of pixels.

An essential part in computing \eqref{eq:target_discrete} is to find the weights $W_\mR(\mR (\i {+} \frac{\t}{T}) {-} \j)$.  For
the nearest neighbour kernel \eqref{eq:nn_kernel}, the integral \eqref{eq:Wd} is simply the intersection area of two squares. As
for the sinc kernel, a formula for \eqref{eq:Wd} can be calculated in the frequency domain using the Parseval's
theorem. Nevertheless, even if large neighbourhoods are avoided by using bounded support kernels, the computation of $W_\mR(\d)$
itself is still costly. One solution is to precompute $W_\mR(\d)$ on a grid of $\mR$ and $\d$ values to look up when necessary.
Most registration algorithms, however, consider a simple form for $W_\mR(\d)$ which does not depend on $\mR$ once $\d$ is
known. Basically, what they do is  discretizing the correlation integral. Let us rewrite
\eqref{eq:target_discrete0} as
\begin{align}
\label{eq:target_discrete2}
\resizebox{.88\columnwidth}{!}{$Q(\mR,\t)  =  \int  \sum_{\k \in P} f_{\k} s(\frac{\x}{T} {-} \k)  \sum_{\j \in P}  \, g_{\j}    \,  s(\frac{\mR (\x {+} \t)}{T} {-} \j) \, d\x.$}
\end{align}
Now, we discretize the above integral at $\x = T \k$ for all $\k \in P$: 
\begin{align}
Q(\mR,\t)  &\approx T^d \sum_{\i \in P}  \sum_{\k \in P} f_{\k} s(\i - \k)  \sum_{\j \in P}  \, g_{\j}    \,  s(\mR (\k + \frac{\t}{T}) - \j) \, d\x,
\nonumber\\
&= T^d \sum_{\i \in P} f_{\i} \sum_{\j \in P}    \, g_{\j}  \, s(\mR (\i + \frac{\t}{T}) - \j) \, d\x. \label{eq:target_discretized}
\end{align} 
Here, the weights are simply the kernel values $s(\mR (\i {+} \frac{\t}{T}) {-} \j)$ which do not depend on $\mR$ if $\mR (\i
{+} \frac{\t}{T}) {-} \j$ is known. Another observation is that \eqref{eq:target_discretized} is independent of what kernel is
used for interpolating $\{f_\i\}$ in the case where $\{f_\i\}$ and $\{g_\i\}$ are interpolated using two different kernels.  For
the nearest neighbour interpolation, $\sum_{\j \in P} \, g_{\j} \, s(\mR (\i + \frac{\t}{T}) - \j)$ is simply the intensity
value of the closest pixel $\j$ to $\mR (\i {+} \frac{\t}{T})$. For the sinc interpolation, this value may be approximated
efficiently by looking up in an upsampled version of $\{g_{\j}\}$. In our experiments we observed that discretized computation
of the correlation integral does not cause major problems, unless for extremely decimated images or when a very high accuracy is
expected (e.g. less than .2 pixels for translation). However, one ideally needs to further consider the discretization
error.


\section{Implications of low resolution registration}
\label{sec:low_res_reg}

Assume that the discrete images $\{f_{\i}\}$ and $\{g_{\i}\}$ are decimated to low-resolution images $\{f^l_{\i}\}$ and
$\{g^l_{\i}\}$. One may ask the question ``What can the fitness \eqref{eq:target_discrete} computed for lower resolution images
say about the fitness for higher resolution images?''. This is an important question since the target function can be computed
much faster in low resolution. If the original images are decimated by a factor of $m$, then computing the target function takes
$1/{m^d}$ less computations, where $d=2,3$ for 2D and 3D images respectively. This fact is clear from
\eqref{eq:target_discrete}, where it is shown that the amount of computations is proportional to the number of
pixels\footnote{In fact, the double summation in \eqref{eq:target_discrete} implies that the amount of computations is
  proportional to the square of the number of pixels. But, in practice, we only consider $\j$-s within a fixed neighbourhood of
  each $\i$.}.

Decimation of a discrete image $\{f_{\i}\}$ may be carried out by
\begin{enumerate}

\item low-pass filtering the corresponding continuous image $f$  obtained from \eqref{eq:interp_sinc}, and

\item sampling the filtered images at a period of $m T$, where $T$ is the sampling period of $\{f_{\i}\}$. 
\end{enumerate}
The low-pass filter handles the aliasing distortion caused by downsampling, and thus, must have a cutoff
frequency of $\frac{1}{2mT}$ or less along every direction. To simplify the derivations, we consider a  radial filter 
with the ideal frequency response
\begin{align}
  \label{eq:low_pass}
  L(\z) = 
   \begin{dcases*}
     1 & $\norm{\z} \leq \frac{1}{2mT}$, \\
     0 & otherwise,
   \end{dcases*}
\end{align}
where $\norm{\z}$ is the $l^2$-norm (length) of $\z$. The filtered image $f^l$ can be obtained as $f^l=l*f$ where $l$ is the
impulse response of $L$ and ``*'' is the convolution operator.  This filter eliminates all the frequency
components outside a ball of radius $\frac{1}{2mT}$ around the origin. Therefore, nothing is lost by sampling the filtered image $f^l$ at intervals
$mT$ to obtain $\{f^l_{\i}\}$, the low resolution image. We can also define a complementary high-pass filter
\begin{align}
  \label{eq:high_pass}
  H(\z) = 
  \begin{dcases*}
      0 & $\norm{\z} \leq \frac{1}{2mT},$ \\
      1 & $ \norm{\z} \geq \frac{1}{2mT},~\norm{\z}_\infty \leq \frac{1}{2T}.$
   \end{dcases*}
\end{align}
The values of $H(\z)$ for $\norm{\z}_\infty {>} \frac{1}{2T}$ do not matter, as $H$ is applied to $f,g$ which are bandlimited to
$\frac{1}{2T}$ due to \eqref{eq:interp_sinc}.  One could filter  $f$ with $H$ to obtain $f^h$. Notice that $f(\x) =
f^l(\x) + f^h(\x)$.  To make a discrete image $\{f^h_{\i}\}$ out of $f^h$ we should sample at intervals $T$. Therefore, the
discrete image $\{f^h_{\i}\}$ has the same size as $\{f_{\i}\}$ while $\{f^l_{\i}\}$ has roughly $1/m$ less samples in every
direction\footnote{We do not sample outside the boundaries of $f_{\i}$ for creating $f^l_{\i}$ and $f^h_{\i}$, even though the
  samples  might not be exactly zero after filtering. The loss is supposedly negligible  given that the images
  have a dark (near zero) margin.}. The corresponding continuous images of $\{f^l_{\i}\}$ and $\{f^h_{\i}\}$ computed from
\eqref{eq:interp_sinc} are exactly equal to $f^l$ and $f^h$ respectively. Note that to use \eqref{eq:interp_sinc} on
$\{f^l_{\i}\}$ one must use $mT$ instead of $T$.

Now, let us have a closer look at the continuous images $f^l$, $f^h$, $g^l$ and $g^h$. An important observation is that
$f^l(\x)$ and $g^h(\mR (\x + \t))$ are orthogonal for any choice of $\mR$ and $\t$. This is because $f^l(\x)$ has no frequency
component outside the ball of radius $\frac{1}{2mT}$, while $g^h(\mR (\x + \t))$ has no frequency components inside this ball. 
One way to see this is to write the inner product (correlation) between these two functions in the frequency domain using the 
Parseval's theorem:
\begin{align}
\label{eq:asdfsafdsa}
Q(\mR,\t) &= \int f^l(\x) \, g^h(\mR (\x {+} \t))\, d\x
\\
&= \int \overline{F^l(\z)} \, G^h(\mR \z) \, e^{2\pi i \t^T \z}\, d\z = 0.
\end{align}
Similarly, $f^h(\x)$ and $g^l(\mR (\x {+} \t))$ are orthogonal. This implies 
\begin{align}
\label{eq:asdfdsagsg}
Q(\mR,\t) &= \int f(\x) \, g(\mR (\x {+} \t))\, d\x 
\nonumber\\
&= \int f^l(\x) \, g^l(\mR (\x {+} \t))\, d\x + \int f^h(\x) \, g^h(\mR (\x {+} \t))\, d\x
\nonumber\\
&= Q^l(\mR,\t) + \int f^h(\x) \, g^h(\mR (\x {+} \t))\, d\x,
\end{align}
where $Q^l(\mR,\t)$ is the target function computed for $f^l$ and $g^l$. Therefore, we have
\begin{align}
| Q(\mR,\t) - Q^l(\mR,\t)| &= \left|\int f^h(\x) \, g^h(\mR (\x {+} \t))\, d\x \right| \label{eq:LH_error_l1}
\\
& \hspace{-80pt} \leq  \sqrt{\int \left(f^h(\x)\right)^2 d\x} ~ \sqrt{\int \left(g^h(\x)\right)^2 d\x} = \norm{f^h}\norm{g^h}, \label{eq:LH_error_l2}
\end{align}
where $\norm{f}$ represents the $L^2$-norm of $f$. We obtain \eqref{eq:LH_error_l2} using the Cauchy-Schwarz inequality followed
by a change of variables $\x \leftarrow \mR (\x {+} \t)$ for the right square root. It follows
\begin{align}
  \label{eq:LH_error}
   Q^l(\mR,\t) - E_{fg}^h \leq Q(\mR,\t)   \leq Q^l(\mR,\t) +  E_{fg}^h
\end{align}
where $E_{fg}^h = \norm{f^h}\norm{g^h}$.  The idea here is that $E_{fg}^h$ tends
to be small as for natural images the energy is mostly concentrated in the lower frequency bands.

Now, suppose that we want to find the maximum of the target function $Q(\mR,\t)$ over a
grid $\{\mR_k,\t_k\}$ of registration parameters. Assume that we have computed the target function $Q^l(\mR,\t)$ for the lower
resolution images for all the grid values $\mR_k,\t_k$. Represent respectively by $\mR^{l*},\t^{l*}$ and $\mR^*,\t^*$ the maximizers of
$Q^l(\mR,\t)$ and $Q(\mR,\t)$ over the grid, where $\mR^*,\t^*$ are yet unknown. Using \eqref{eq:LH_error} we get
\begin{align}
  Q(\mR^*,\t^*) \geq Q(\mR^{l*},\t^{l*}) \geq Q^l(\mR^{l*},\t^{l*}) - E_{fg}^h.
\end{align}
It means that we could rule out any $\mR,\t$ for which $Q^l(\mR,\t) 
< Q^l(\mR^{l*},\t^{l*}) - 2 E_{fg}^h$, since in that case 
\begin{align}
  Q(\mR,\t) &\leq Q^l(\mR,\t) + E_{fg}^h 
  \nonumber\\
  & < Q^l(\mR^{l*},\t^{l*}) - E_{fg}^h \leq Q(\mR^*,\t^*)
\end{align}
Fig. \ref{fig:lh_res_sinc}(c) illustrates this idea.  To further limit the search space one can compute the high resolution
target function $Q$ at $\mR^{l*},\t^{l*}$ and discard all $\mR,\t$ with $Q^l(\mR,\t) < Q(\mR^{l*},\t^{l*}) - E_{fg}^h$, as shown
in Fig.~\ref{fig:lh_res_sinc}(d). The example provided in Fig.~\ref{fig:lh_res_sinc}  shows the effectiveness of the proposed bound
even when the images are decimated by a factor of 16 in each dimension.

\begin{figure*}[float]
  \centering
  \begin{tabular}[c]{cc}
    \includegraphics[width=0.28\textwidth]{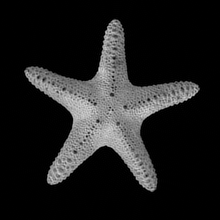}
    \includegraphics[width=0.0175\textwidth]{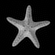}
    &
    \includegraphics[width=.28\textwidth]{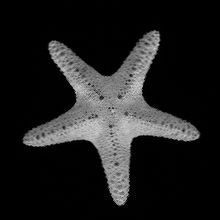}
    \includegraphics[width=0.0175\textwidth]{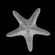}
    \\
    (a)
    &
    (b)
    \\
    \includegraphics[width=.28\textwidth]{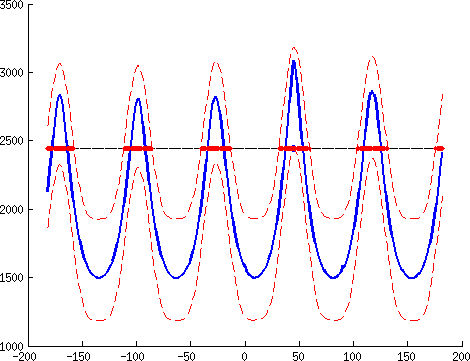}
    &
    \includegraphics[width=.28\textwidth]{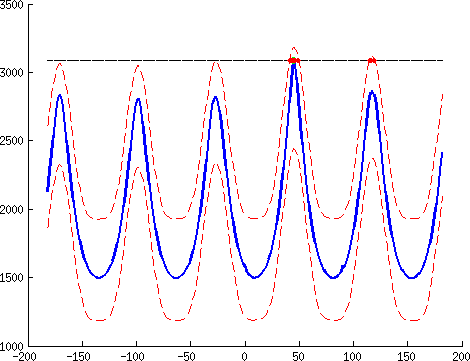}
    \\
    (c)
    &
    (d)
  \end{tabular}
  \caption{Bounds on high resolution registration fitness based on low resolution registration.  (a) A starfish image (obtained
    from http://www.ck12.org) is chosen as the reference image $\{f_\i\}$, as it exhibits multiple local maxima, (b) $\{f_\i\}$
    is rotated by 45 degrees around its centre and then undergoes a slight affine warp to produce $\{g_\i\}$. The only
    registration parameter is the rotation angle $\theta$. Both images are decimated by a factor of $m=16$ to create
    $\{f^l_\i\}$ and $\{g^l_\i\}$ (tiny images on the bottom-right of each image). (c) the lower and upper bounds
    $Q^l(\theta) \pm E_{fg}^h$ (red dashed graphs) envelop the high resolution fitness $Q(\theta) = Q(\mR,\t)$ (blue solid
    graph).  The horizontal line shows the value $Q^l(\mR^{l*},\t^{l*}) - E_{fg}^h$. The red thick parts of this line show the
    reduced search space, where we have $Q^l(\mR,\t) + E_{fg}^h \geq Q^l(\mR^{l*},\t^{l*}) - E_{fg}^h$ or
    $Q^l(\mR,\t) \geq Q^l(\mR^{l*},\t^{l*}) - 2 E_{fg}^h$. (d) to further restrict the search space the high resolution fitness
    $Q$ is computed at $(\mR^{l*},\t^{l*}) \equiv \theta^{l*}$. The horizontal line shows the value $Q(\mR^{l*},\t^{l*})$. The
    red thick parts of this line shows the reduced search space, namely where $Q^l(\mR,\t) + E_{fg}^h \geq
    Q(\mR^{l*},\t^{l*})$.
    Notice that decimating by a factor of 16 reduces the computations by a factor of 256 for 2D images. The amount of reduction
    in the search space provided by such a low resolution registration is notable.}
  \label{fig:lh_res_sinc}
\end{figure*}

\subsection{Bounded support interpolation}
\label{sec:bounded_support}
The results in the previous section hold when the sinc kernel \eqref{eq:sinc} is used to compute the target functions
\eqref{eq:target_discrete} or \eqref{eq:target_discretized}.  In this case, $s(\d)$ and $W_\mR(\d)$ do not have a compact
support, and for \eqref{eq:LH_error} to hold we need a large neighbourhood of pixels $\j$ around the location of each
transformed pixel $\mR (\i {+} \frac{\t}{T})$, which requires a lot of computations. We now consider computing the target
function \eqref{eq:target_discrete} with two kernels of bounded support, namely  the box
kernel \eqref{eq:nn_kernel} and the triangular kernel \eqref{eq:bilinear_kernel} which respectively correspond to nearest neighbour
and bilinear (or trilinear) interpolation methods.  Many other interpolation techniques can be treated  in a similar
manner.

Here, we assume that the low resolution images are obtained, as before, by performing the ideal filter \eqref{eq:low_pass} on
the continuous images $f$ and $g$ obtained by the \emph{sinc} kernel using \eqref{eq:interp_sinc}, and then resampling. However,
for computing the target function using \eqref{eq:target_discrete}, we use a box or triangular kernel for both low-resolution
and high-resolution images. Now, let us see what happens in the frequency domain.  The interpolation formula \eqref{eq:interp}
can be written as the following convolution
\begin{align}
  \label{eq:interp_conv}
  f'(\x) =  s(\frac{\x}{T}) * \scalebox{1.4}{$($} f(\x) \cdot \sum_{\i \in \Zahlen^d} \delta(\x - T \i ) \scalebox{1.4}{$)$},
\end{align}
where $\delta$ is the Dirac delta function. The Fourier transform of the kernel $s(\x)$ is $\sinc^\alpha(\z)$ with
$\alpha=1$ for the box kernel \eqref{eq:nn_kernel} and $\alpha=2$ for the triangular kernel \eqref{eq:bilinear_kernel}.  The
Fourier transform of \eqref{eq:interp_conv} is then
\begin{align}
  \label{eq:fourier_interp_freq}
  F'(\z) =  \sinc^\alpha(T \z) \cdot \sum_{\i \in \Zahlen^d} F(\z - \frac{\i}{T}).
\end{align}
Notice that $F(\z)$ is bandlimited to $\frac{1}{2T}$ in every dimension. Therefore, the term
$\sum_{\i \in \Zahlen^d} F(\z - \frac{\i}{T})$ is just the periodic repetition of $F(\z)$ with a period of $\frac{1}{T}$
along every dimension.  In the same way, for the low-resolution image $\{f^l_\i \}$ we have
\begin{align}
  \label{eq:fourier_interp_freq_low_res}
  F'^l(\z) =  \sinc^\alpha(mT \z) \cdot \sum_{\i \in \Zahlen^d} F^l(\z - \frac{\i}{mT}).
\end{align}
Notice that  $F^l(\z) = F(\z) \cdot L(\z)$ is bandlimited to $\frac{1}{2mT}$ along every
dimension. Similarly, $G'(\z)$ and $G'^l(\z)$ can be obtained.

Now, we like to find a bound on $\left|Q'(\mR,\t) - Q'^l(\mR,\t) \right|$, where $Q'(\mR,\t)$ is the target function
\eqref{eq:target} calculated for $f'$ and $g'$, and $Q'^l(\mR,\t)$ is the target for $f'^l$ and $g'^l$. Here, we present a simple
way of doing this, and leave more elaborate bounds for future research. Define the energy of a function $F$ within the
frequency region $S$ by
\begin{align}
  E_S(F) = \int_S |F(\z)|^2 d\z,
\end{align}
denote by $\Omega$ the ball of radius $\frac{1}{2mT}$ around the origin in the frequency domain, and let $\bar{\Omega}$ be
its complement.  Then we have the following proposition:
\begin{proposition}
  \label{thm:bound_supp_interp}
  The absolute difference on the target functions $\left|Q'(\mR,\t) - Q'^l(\mR,\t) \right|$ is
  bounded from above by
  \begin{align}
    &\sqrt{ E_{\Omega}(F') \, E_{\Omega}(G' {-} G'^l)}
    + \sqrt{ E_{\Omega}(G') \, E_{\Omega}(F' {-} F'^l)} \label{eq:bound_supp_interp_in1}
    \\
     +  &\sqrt{ E_{\Omega}(F' {-} F'^l) \, E_{\Omega}(G' {-} G'^l)} \label{eq:bound_supp_interp_in2}
    \\
     + &\sqrt{E_{\bar{\Omega}}^{\phantom{l}}(F') \, E_{\bar{\Omega}}(G')} + \sqrt{E_{\bar{\Omega}}(F'^l) \, E_{\bar{\Omega}}(G'^l)} \label{eq:bound_supp_interp_out}
   \end{align}
\end{proposition}
Notice that outside the frequency ball $\Omega$ one expects $F',G', F'^l$ and $G'^l$ to have low energy and hence
\eqref{eq:bound_supp_interp_out} is supposed to be small.  As for \eqref{eq:bound_supp_interp_in1} and
\eqref{eq:bound_supp_interp_in2}, it is expected that within $\Omega$, the energies $E_{\Omega}(F' {-} F'^l)$ and $E_{\Omega}(G'
{-} G'^l)$ are small, as we will soon see. The proof is  is given in Appendix
\ref{sec:bounded_support_interp_app}. Here, we give the value of the above energy terms calculated for function $F$.
\begin{align}
  E_{\Omega}(F')& = \int_\Omega \left|\sinc^\alpha( T  \z) \right|^2 |F(\z)|^2 d\z, \label{eq:Eo_Fp}
\\
  E_{\Omega}(F' {-} F'^l) &= \int_\Omega \resizebox{.41\columnwidth}{!}{$\left|\sinc^\alpha( T  \z) {-} \sinc^\alpha(m T \, \z)\right|^2$}  |F(\z)|^2 d\z \label{eq:Eo_dFp}
\\
  E_{\bar{\Omega}}(F')  &= \int_{C \setminus \Omega} \sinc^{2\alpha}(T \z) \, |F(\z)|^2 d\z \label{eq:Eob_Fp1}
  \\
  &\hspace{-4pt}+ \int_{C} \left(\Phi_\alpha(T\z) - \sinc^{2\alpha}(T \z)\right) \, |F(\z)|^2 d\z, \label{eq:Eob_Fp2}
\end{align}
\begin{align}
  E_{\bar{\Omega}}(F'^l) &= \int_{\Omega} \left(\Phi_\alpha(m T \z) {-} \sinc^{2\alpha}(m T \z)\right) \, |F(\z)|^2 d\z, \label{eq:Eob_Flp}
\end{align}
where $C$ is the $l^\infty$-ball of radius $\frac{1}{2T}$, that is $\{\z \,|\, |z_i| \leq \frac{1}{2T}\}$, and the function 
$\Phi_\alpha$ is defined on the $l^\infty$-ball of radius $\frac{1}{2}$ as
\begin{align}
  \Phi_\alpha(\z) = \sum_{\i \in \Zahlen^d } \sinc^{2\alpha}(\z + \i).
\end{align}
For the 1D case ($d=1$) it has the following formula
\begin{align}
  \Phi_\alpha(z) = & \sinc^{2\alpha}(z) (1 +  \frac{\psi_{2\alpha{-}1}(1{+}z) + \psi_{2\alpha{-}1}(1{-}z)}{z^{-2\alpha}(2\alpha{-}1)!} ),
\end{align}
for $-\frac{1}{2} \leq z \leq \frac{1}{2}$, where  $\psi_k$ is the polygamma function of order $k$. For $d > 1$ we define  $\Phi_\alpha(\z) = \prod_{i=1}^d
\Phi_\alpha(z_i)$. Considering Proposition \ref{thm:DFT_samples} the
 integrals (\ref{eq:Eo_Fp}-\ref{eq:Eob_Flp}) can be computed numerically using DFT.

 It is simple to derive \eqref{eq:Eo_Fp} and \eqref{eq:Eo_dFp} from \eqref{eq:fourier_interp_freq} and
 \eqref{eq:fourier_interp_freq_low_res} considering the fact that $\sum_{\i \in \Zahlen^d} F(\z - \frac{\i}{T})$ and $\sum_{\i
   \in \Zahlen^d} F^l(\z - \frac{\i}{m T})$ are both equal to $F(\z)$ inside $\Omega$. Notice that \eqref{eq:Eo_dFp} tends to be
 relatively small for natural images as for small sized $\z$ the difference $|\sinc^\alpha( T \z) - \sinc^\alpha(m T \, \z)|$ is
 small and for larger $\z$ the frequency spectrum $F(\z)$ tends to be small. With a similar argument one can assert that the integrals \eqref{eq:Eob_Fp2}
 and \eqref{eq:Eob_Flp} are small. The integral \eqref{eq:Eob_Fp1} is small as it is over $C \setminus \Omega$. 
To obtain \eqref{eq:Eob_Fp1} and \eqref{eq:Eob_Fp2} observe that $E_{\bar{\Omega}}(F')$ is equal to
\begin{align}
  \int_{C \setminus \Omega} |F'(\z)|^2 d\z + \int_{\bar{C}} |F'(\z)|^2 d\z
\end{align}
As  $F'(\z) = \sinc^\alpha(T \, \z) F(\z)$ for all $\z \in C$, the  integral over $C \setminus \Omega$ is equal to \eqref{eq:Eob_Fp1}. For the
integral over $\bar{C}$, first for a vector $\v$ and a set $S$ define $\v {+} S \eqdef \{\v {+} \z \,|\, \z \in S \}$. Then, we have
\begin{align}
  \int_{\bar{C}} |F'(\z)|^2 d\z &= \sum_{\i \in \Zahlen^d \setminus \Zero} \int_{\i/T + C} |F'(\z)|^2 d\z
  \nonumber\\
                                &=  \sum_{\i \in \Zahlen^d \setminus \Zero}  \int_{C} |F'(\z + \i/T )|^2 d\z
  \nonumber\\
  & =  \sum_{\i \in \Zahlen^d \setminus \Zero}  \int_C |F(\z)|^2 \sinc^2(T \z + \i )  d\z,
\end{align}
which is equal to \eqref{eq:Eob_Fp2} as
$\sum_{\i \in \Zahlen^d \setminus \Zero} \sinc^2(T \z + \i ) = \Phi_\alpha(T\z) - \sinc^{2\alpha}(T \z)$. In a similar way
\eqref{eq:Eob_Flp} can be obtained; only the period is $\frac{1}{mT}$ instead of $\frac{1}{T}$. This means that, instead of $C$,
we must consider the $l^\infty$ ball $C'$ of radius $\frac{1}{mT}$.  Notice that in this case
$\int_{C' \setminus \Omega} |F'^l(\z)|^2 d\z$ is zero, as $F^l(\z) = F(\z) \cdot L(\z)$ is zero outside $\Omega$.  All the
integrals can be numerically computed using DFT in the light of Proposition~\ref{thm:DFT_samples}. Fig.~\ref{fig:lh_res_nn}(a-c)
illustrates the effect of the bound (\ref{eq:bound_supp_interp_in1}-\ref{eq:bound_supp_interp_out}) for the example of Fig. \ref{fig:lh_res_sinc}.

\begin{figure*}[float]
  \centering
  \begin{tabular}[c]{ccc}
    \includegraphics[width=.24\textwidth]{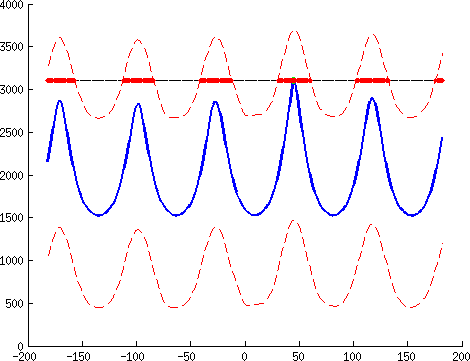}
    &
    \includegraphics[width=.24\textwidth]{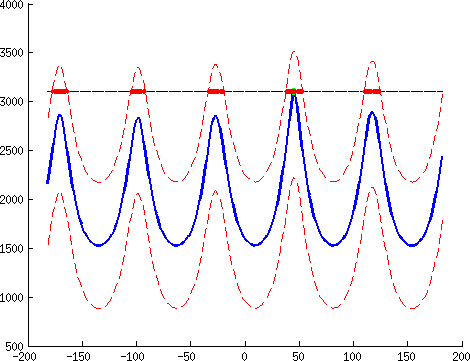}
    &
    \includegraphics[width=.24\textwidth]{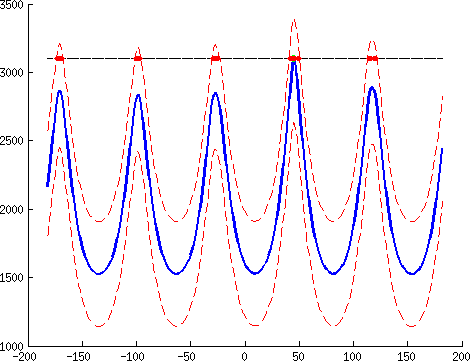}
    \\
    (a) m = 16
    &
    (b)  m = 8
    &
    (c)  m = 4
    \\
    \includegraphics[width=.24\textwidth]{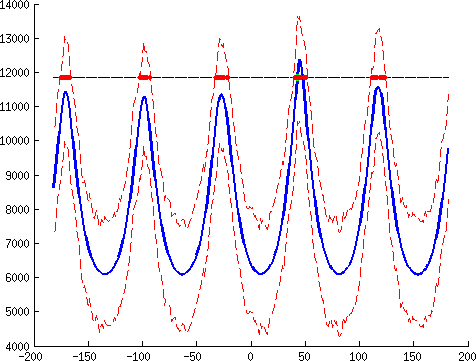}
    &
    \includegraphics[width=.24\textwidth]{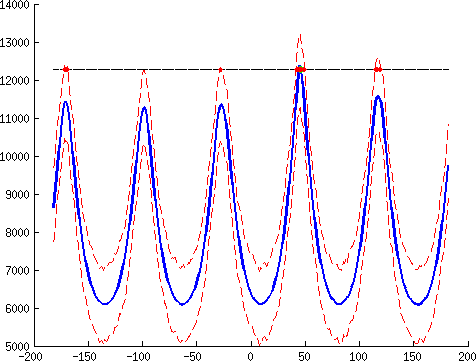}
    &
    \includegraphics[width=.24\textwidth]{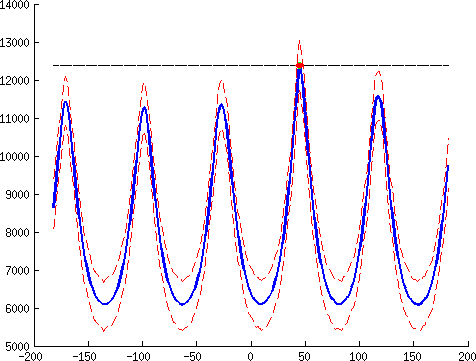}
    \\
    (d) m = 16
    &
    (e)  m = 8
    &
    (f)  m = 4
  \end{tabular}
  \caption{The bounds obtained for the nearest neighbour interpolation for the registration problem given in
    Fig.~\ref{fig:lh_res_sinc}. Images are decimated by a factor of (a) 16, (b) 8 and (c) 4. The search space reduction approach
    is the same as that of Fig.~\ref{fig:lh_res_sinc}(d). The bounds are not as effective as that of the sinc interpolation, but
    still are useful. (d,e,f) the same figures, but this time discretized integration is used instead of exact integration to
    compute the target function, and \eqref{eq:bound_supp_interp_one_sinc2} is used for the bounds. Notice that the discretized
    integration has resulted in fluctuations in the bounds which increase as the decimation rate $m$ gets larger.}
  \label{fig:lh_res_nn}
\end{figure*}

Finally, let us consider the case where the discretized integration \eqref{eq:target_discretized} is used for computing the
target function. As  discussed earlier, this approximation does not depend on how $\{f_\i\}$ is interpolated. Therefore, we
may assume that $\{f_\i\}$ is interpolated with a sinc kernel, that is $F'{=}F$ and $F'^l {=} F^l$. Under these assumptions, one can show that 
$\left|Q'(\mR,\t) - Q'^l(\mR,\t) \right|$ is bounded from above by
\begin{align}
  \sqrt{ E_{\Omega}(F) \, E_{\Omega}(G' {-} G'^l)}
  + \sqrt{E_{\bar{\Omega}}^{\phantom{l}}(F) \, E_{\Omega' \setminus \Omega}(G')}  \label{eq:bound_supp_interp_one_sinc2}
\end{align}
where $\Omega'$ is the ball of radius $\frac{\sqrt{d}}{2T}$ (see Appendix \ref{sec:bounded_support_interp_app_Fsinc}).  The
effect of using this bound is depicted in Fig.~\eqref{fig:lh_res_nn}(d,e,f). The bound is generally smaller than
(\ref{eq:bound_supp_interp_in1}-\ref{eq:bound_supp_interp_out}). However, here the discretization error for computing the target integral has been
neglected. This is evident from the slight fluctuations in the bounds in Fig.~\eqref{fig:lh_res_nn}(d,e) for very low resolutions.


\subsection{Low-resolution to high-resolution registration}
\label{sec:low_high_reg}
In this section we examine the situation where only one of the images $\{f_\i\}$ is decimated. We study how this affects the amount of
computations and the quality of the bounds. We also consider the case where the other image $\{g_\i\}$ is upsampled.

Assume that  $\{f_\i \}$ is decimated by a factor of $m$ to create the low-resolution image $\{f^l_\i\}$. Using 
\eqref{eq:interp} a corresponding continuous image can be obtained as 
\begin{align}
  \label{eq:interp_low_res}
  f'^l(\x) = \sum_{\i \in P^l} f^l_{\i} ~ s(\frac{\x}{mT} - \i),
\end{align}
where $P^l$ is the grid of low-resolution pixels. For the image $\{g_\i \}$ we do not change the resolution, and thus, we have
$g'(\x) = \sum_{\i \in P} g_{\j} ~ s(\frac{\x}{T} - \j)$. Now, there are two ways to compute the correlation target between
$f'^l(\x)$ and $g'(\mR(\x{+}\t))$: the exact way like \eqref{eq:target_discrete} and the approximate way by discretizing the
integral like \eqref{eq:target_discretized}. We leave it to the reader to check that for computing the exact integral similar to
\eqref{eq:target_discrete}, lowering the resolution of just one image does not significantly reduce the  computations. This is because
while the number of pixels in the first image is reduced by a factor of $m^d$, the size of the neighbourhood  around
each transformed pixel required for the computation of $W_\mR(\d)$ is increased by the same factor. However, if we discretize the integral, like in \eqref{eq:target_discretized} 
we get 
\begin{align}
Q'^{hl}(\mR,\t) &=  \int f'^l(\x) \, g'(\mR (\x + \t))\, d\x \label{eq:target_hl}
\\
&\hspace{-6pt}\approx (m T)^d \sum_{\i \in P^l} f_{\i} \sum_{\j \in P}    \, g_{\j}  \, s(\mR (m \i + \frac{\t}{T}) - \j) \, d\x. \label{eq:target_hl_discretized}
\end{align}
The above shows that if $s$ has a compact support then for every $\i$ we only need to consider those pixels $\j$ which are in
the corresponding neighbourhood of $\mR (m \i + \frac{\t}{T})$. The size of this neighbourhood is the same as that of 
 \eqref{eq:target_discretized}. Since the first sum is over the low-resolution pixel grid $P^l$ which has
about $m^d$ times less pixels than $P$, we see that the amount of required computation when lowering the resolution 
of just one image is the same as when both images are decimated. This, of course, comes at the cost of having an approximate integral
by discretization. 

Now, let us see what happens to the bounds. If the interpolation is done using the sinc kernel
($s = \sinc$, and thus $f'^l {=} f^l$ and $g' {=} g$), then  the target function \eqref{eq:target_hl} is equal to
$Q^l(\mR,\t) {=} \int f^l(\x) \, g^l(\mR (\x + \t))\, d\x$, and hence, the bound \eqref{eq:LH_error} does not change. This is due
to the fact that $f^l(\x)$ and $g^h(\mR (\x + \t))$ are orthogonal, and thus, the correlation between $f^l(\x)$ and $g(\mR (\x +
\t)) = g^l(\mR (\x + \t)) + g^h(\mR (\x + \t))$ is the same as the correlation between $f^l(\x)$ and $g^l(\mR (\x + \t))$. 
However, for other interpolation techniques the bounds can be further tightened. In a similar way to Proposition \ref{thm:bound_supp_interp}
we can obtain
\begin{align}
  \left|Q'(\mR,\t) - Q'^{hl}(\mR,\t) \right| & \leq \sqrt{ E_{\Omega}(G') \, E_{\Omega}(F' {-} F'^l)} 
  \nonumber\\
   & \hspace{-2cm}+ \sqrt{E_{\bar{\Omega}}^{\phantom{l}}(F') \, E_{\bar{\Omega}}(G')} + \sqrt{E_{\bar{\Omega}}(F'^l) \, E_{\bar{\Omega}}(G')}, \label{eq:bound_supp_interp_hl_out}
\end{align}
which is generally smaller than the bound in Proposition~\ref{thm:bound_supp_interp}. If the discretized integral is used, similar to the 
way \eqref{eq:bound_supp_interp_one_sinc2} is obtained, one can obtain the following bound
\begin{align}
  \sqrt{E_{\bar{\Omega}}^{\phantom{l}}(F) \, E_{\Omega' \setminus \bar{\Omega}}(G')},  \label{eq:bound_supp_interp_one_sinc_hl2}
\end{align}
which is obtained by replacing $G'^l$ with $G'$ in \eqref{eq:bound_supp_interp_one_sinc2}. Fig.~\ref{fig:h_lh_res_nn}(a)
shows an example of applying this bound. 

Now, let us see what happens if the second image $\{g_\i\}$ is upsampled by a factor of $p \geq 2$ to obtain $\{g^u_\i\}$.  We
assume that the new samples are obtained by the natural sinc interpolation. Therefore, the corresponding continuous function remains 
the same, which means $g^u(\x)=g(\x)$. Similarly to the way we obtained
\eqref{eq:fourier_interp_freq}, if  $\{g^u_\i\}$ is interpolated by the nearest neighbour ($\alpha {=} 1$) or
the bilinear ($\alpha {=} 2$) method  to obtain the interpolated image $g'^u$, the corresponding Fourier transform will be
\begin{align}
  \label{eq:fourier_interp_freq_upsampled}
  G'^u(\z) =  \sinc^\alpha(\frac{T \z}{p}) \cdot \sum_{\i \in \Zahlen^d} G(\z - \frac{p \i}{T}).
\end{align}
Notice that  $\sum_{\i \in \Zahlen^d} G(\z - \frac{p \i}{T})$ is the periodic repetition of  $G(\z)$  with  period 
$p/T$. Since $G(\z)$ is bandlimited to $\frac{1}{2T}$ in every dimension, this means that for $p \geq 2$ there exist a lot of
\emph{empty spaces} in the spectrum \eqref{eq:fourier_interp_freq_upsampled} in which $G'^u(\z) = 0$.  Specially, $G'^u(\z)$ is zero when
$\frac{1}{2T} {<} \norm{\z}_\infty {<}\frac{p}{2T}$. Thus, $G'^u(\z) {=} 0$ inside
$\Omega' {\setminus} C$, where $C$ is the $l^\infty$-ball of radius $\frac{1}{2T}$, and $E_{\Omega' \setminus {\Omega}}(G'^u) =
E_{C \setminus {\Omega}}(G'^u)$. Now, if we derive \eqref{eq:bound_supp_interp_one_sinc_hl2} for $G'^u(\z)$ instead of $G'(\z)$
we will get the following bound
\begin{align}
  \sqrt{E_{\bar{\Omega}}^{\phantom{l}}(F) \, E_{C \setminus {\Omega}}(G'^u)}.  \label{eq:bound_supp_interp_one_sinc_hl_upsample}
\end{align}
Notice that for all $\z \in C$ we have 
\begin{align}
  G'^u(\z) = \sinc^\alpha({T \z}/{p}) \, G(\z).
\end{align}
Therefore, inside $C$, $G'^u(\z)$ approaches to $G(\z)$ as $p$ gets larger. This means that as 
$p$ increases,  \eqref{eq:bound_supp_interp_one_sinc_hl_upsample} gets closer  to the original bound $E_{fg}^h$  in
\eqref{eq:LH_error_l2}, since $E_{\bar{\Omega}}^{\phantom{l}}(F) = \int_{\bar{\Omega}} |F(\z)|^2 d\z = \int
\left(f^h(\x)\right)^2 d\x$. Nonetheless, one should bear in
mind that here we have neglected the discretization error for computing the target integral\footnote{The reader might have
  notice that the bound \eqref{eq:bound_supp_interp_one_sinc_hl_upsample} is actually smaller than \eqref{eq:LH_error_l2}. This,
  however, this does not make it a better bound, since the corresponding target functions are
  different.}. Fig.~\ref{fig:h_lh_res_nn}(b,c) shows examples of using \eqref{eq:bound_supp_interp_one_sinc_hl_upsample} for $p=2$ and $p=4$. We have observed
experimentally that upsampling by a factor larger than 2 does not help  much with reducing the search space. We, therefore,
use $p=2$ in all our experiments.

\begin{figure*}[t!]
  \centering
  \begin{tabular}[c]{ccc}
    \includegraphics[width=.24\textwidth]{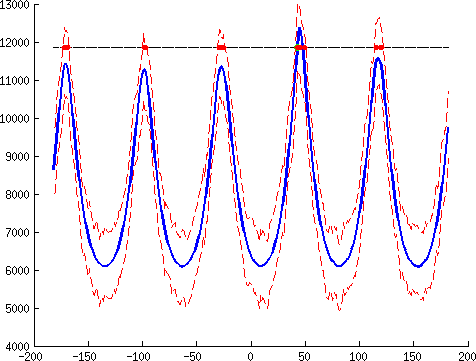}
    &
    \includegraphics[width=.24\textwidth]{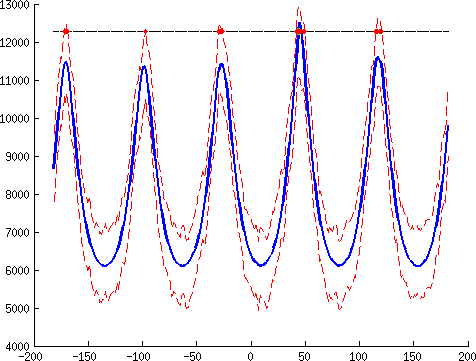}
    &
    \includegraphics[width=.24\textwidth]{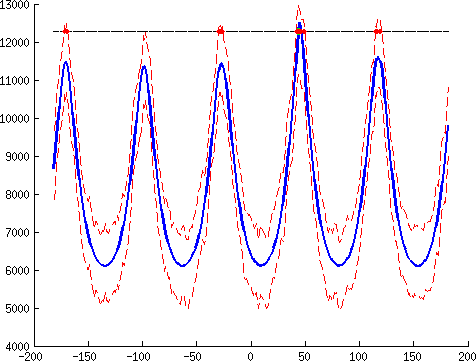}
    \\
    (a) m = 16, p = 1
    &
    (b) m = 16, p = 2
    &
    (c) m = 16, p = 4
  \end{tabular}
  \caption{Bounds for low-resolution to high-resolution registration with nearest neighbour interpolation and discretized
    computation of the target integral for the registration problem of Fig.~\ref{fig:lh_res_sinc}. In all cases, the image $f$
    is decimated by a factor of $m=16$, and $g$ has been upsampled by a factor of (a) $p=1$, (b) $p=2$ and (c) $p=4$. Changing
    $p$ from 2 to 4 has not made a noticeable difference.}
  \label{fig:h_lh_res_nn}
\end{figure*}

\subsection{Obtaining Tighter bounds}
\label{sec:tighter}

\subsubsection{Partitioning to radial frequency bands}
\label{sec:radial_freq_bands}
We can obtain slightly tighter bounds by dividing the high-frequency area into radial bands.  From \eqref{eq:LH_error_l1} we have
\begin{align}
| Q(\mR,\t) - Q^l(\mR,\t)| = \left|\int_{\bar{\Omega}} F(\z) \, G(\mR \z (\x)) \, e^{2\pi i \t^T \x}\, d\z \right| 
\end{align}
where $\bar{\Omega}$ is the area outside the ball of radius $\frac{1}{2mT}$ around the origin. The above is obtained using the
Parseval's theorem given the fact that $F^h(\z) {=} H(\z)\,F(\z)$ and $G^h(\z) {=} H(\z)\,G(\z)$ where the high-pass filter $H$ was
defined in \eqref{eq:high_pass}.  Now, we partition $\bar{\Omega}$ to radial bands $\Omega_1, \Omega_2, \ldots, \Omega_P$ as
illustrated in Fig. \ref{fig:nested}. More precisely, $\Omega_i = \{\z \,|\, r_i \leq \norm{\z} < r_{i+1}\}$ for radii $r_1 <
r_2 < \cdots < r_{P+1}$. Then we have
\begin{align}
| Q(\mR,\t) - Q^l(\mR,\t)| & = \left| \sum_{k=1}^P \int_{\Omega_k} F(\z) \, G(\mR \z (\x)) \, e^{2\pi i \t^T \x}\, d\z \right| 
\nonumber\\
&  \hspace{-40pt} \leq \sum_{k=1}^P {\textstyle \sqrt{\int_{\Omega_k} |F(\z)|^2 d\z} \sqrt{\int_{\Omega_k} |G(\z)|^2 d\z}} \label{eq:tight_freq}
\end{align}
where \eqref{eq:tight_freq} is obtained by using the triangle inequality followed by the Cauchy-Schwarz inequality, and a change
of variables $\z \leftarrow \mR\z$ in the second integral. Notice that $\mR\z \in \Omega_k$ if and only if $\z \in \Omega_k$.
The bound \eqref{eq:tight_freq} can be computed numerically using DFT (see Proposition~\ref{thm:DFT_samples}).  We have tighter
bounds if more frequency bands are used. The extreme case is when each band only contains DFT samples with the same distance to
the origin.  On the other extreme, when $P=1$, then \eqref{eq:tight_freq} is reduced to \eqref{eq:LH_error_l2}. Our experiments
show that for natural image pairs with similar spectra \eqref{eq:tight_freq} is only slightly better than \eqref{eq:LH_error_l2}. For
instance, in the example of Fig.~\ref{fig:lh_res_sinc}, the bound \eqref{eq:LH_error_l2} is equal to 372.6 while
\eqref{eq:tight_freq} is equal to 356.4 in the best case (maximum $P$). However, dividing the frequency domain into radial
bands can be a useful trick in similar problems. We will use it for obtaining Lipschitz bounds in
Sect. \ref{sec:lipschitz_bounds}.

\begin{figure}[t!]
  \centering
  \begin{center}
    \includegraphics[width=.24\textwidth]{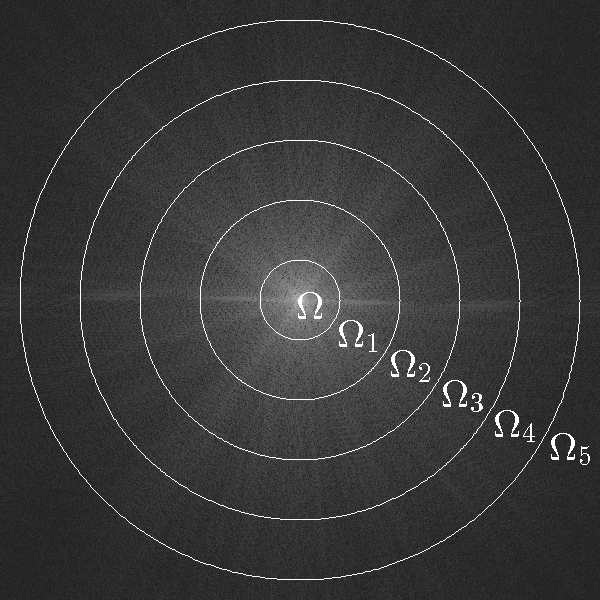}
  \end{center}
  \caption{The frequency domain outside $\Omega$  has been partitioned into nested radial frequency bands   $\Omega_1, \Omega_2, \ldots,
\Omega_P$ ($P=5$). 
}
  \label{fig:nested}
\end{figure}

\subsubsection{Hölder's inequality}
Instead of using the Cauchy-Schwarz inequality $\langle f,g \rangle \leq \norm{f} \norm{g}$ in the derivation of the bounds, one may
use the more general Hölder's inequality which asserts $\langle f,g \rangle \leq \norm{f}^\alpha \norm{g}^\beta$ for any
$\alpha,\beta \in [1,\infty]$ such that $\frac{1}{\alpha} + \frac{1}{\beta} = 1$. This can be done by minimizing
$\norm{f}^\alpha \norm{g}^\beta$ with respect to $\alpha$ and $\beta$. This is specially useful when of one the images is
considerably smaller or sparser than the other.

\section{A basic algorithm}
\label{sec:basic_algorithm}
We showed that the search space for high-resolution images can be limited given the low-resolution target values. This inspires
a multiresolution search strategy whereby the target values at each resolution level further restrict the search space for the
next (higher) level of resolution, as formalized in Algorithm~\ref{alg:multi_res}. This is a very basic algorithm which searches
a grid of the registration parameters $\tG$. Here, we use $l$ as an index to represent the \emph{resolution level}, as opposed
to the previous section, where $l$ was a symbol indicating \emph{low-resolution}. The resolution levels are
$l= 0,1,\ldots,l_{\text{max}}$, where $l{=}0$ and $l_{\text{max}}$ respectively correspond to highest and lowest
resolutions. The target function for each resolution is represented by $Q^l$. Therefore, $Q^0(\mR,\t) = Q(\mR,\t)$.  At each
resolution level, the algorithm finds the maximizer $(\mR^{l*},\t^{l*})$ of the target function $Q^l$ over $\tG$ (line
\ref{alg:line:max}), updates $Q^*$, the so-far best target value (line \ref{alg:line:update_Qs}), and discards some elements of
$\tG$ based on $Q^*$ and the inter-resolution bounds $B_{\text{res}}^l$ (line \ref{alg:line:rule_out}). Here, $B_{\text{res}}^l$
can be any of the bounds introduced in the previous section. The algorithm might not seem interesting from a practical
perspective as the target function has to be evaluated at every element of the grid, at least at the lowest
resolution. Nonetheless, it demonstrates a strategy purely based on our inter-resolution bounds, and gives insights about how to
reduce computations in globally optimal registration algorithms where the whole parameter space is searched, which is the
subject of the next section.

\begin{algorithm}
  \caption{A basic multiresolution rigid motion space search algorithm for intensity-based registration. }\label{alg:multi_res}
  
  \begin{algorithmic}[1]
    \Require 
    \Statex $\{f_\i^l\},\{g_\i^l\}$: registration image pairs for each resolution $l= 0,1,\ldots,l_{\text{max}}$,
    \Statex $B^l_{\text{res}}$: inter-resolution bounds for each $l$,
    \Statex $\tG$: a search grid  of parameters $(\mR,\t)$,
    \Statex $m_0 {<} m_1 {<} \cdots {<} m_{l_{\text{max}}}$: A set of decimation rates ($m_0{=}1$),
    \Statex $\mR_0,\t_0$:  an initial solution (optional).  
    
    \State $Q^* \gets Q(\mR_0,\t_0)$ or $-\infty$   \Comment{so-far best target value}
    \For{$l \gets l_{\text{max}} $ \textbf{downto} $0$}  \Comment{low to high resolution}
    \State $\mR^{l*},\t^{l*} \gets \arg\max_{\mR,\t \in \tG} Q^{l}(\mR,\t)$\label{alg:line:max}
    \State $Q^* \gets \max(Q^*, Q(\mR^{l*},\t^{l*}))$ \label{alg:line:update_Qs}
    \State $\tG \gets \tG \setminus \{\mR,\t \in \tG \,|\,  Q^{l}(\mR,\t) < Q^* {-} B_{\text{res}}^l\}$ \label{alg:line:rule_out}

     \EndFor
     \State \Return $\mR^*,\t^*$

    \end{algorithmic}
  \end{algorithm}

  We test Algorithm~\ref{alg:multi_res} for the registration of the image pairs in Fig.~\ref{fig:lh_res_sinc}.  First, we
  consider only the rotation parameter using a grid of rotation angles between -180 to 180 degrees with 0.1 degree intervals.
  Six levels of resolution has been used with decimation factors 32, 16, 8, 4, 2 and 1.  For further implementation details we
  refer the reader to Sect. \ref{sec:technical}.  It takes about 0.2 seconds for the algorithm to find the best grid
  element. The second column of Table \ref{tab:computation_per_res} shows the percentage of grid points ruled out at each
  resolution level by the bounds (in line \ref{alg:line:rule_out} of Algorithm~\ref{alg:multi_res}). It is interesting that more
  than 99 percent of the grid elements are thrown away at the lowest  resolution where the images are decimated by a
  factor of 32, and thus, have  about 1000 times less pixels. It is also evident from the last row of the table that the best grid
  point has been already found at the second-highest level of resolution ($m=2$). The third column of Table
  \ref{tab:computation_per_res} shows the same figures for when translation is also included. The image $g$ undergoes a translation
  of $(10,40)$ pixels on top of the current rotation. Here, the rotation grid has 1 degree intervals  and the translation
  grid is in polar coordinates where the angle of translation has 1 degree intervals and the length of translation is half of
  the image size in each direction with intervals of 1 pixel. It took 31 minutes  to search the whole
  grid. Here, 99.9 percent of the grid values were ruled out at the lowest resolution level, which means a huge decrease in
  computations. We should mention, however, that for the image of Fig.~\ref{fig:lh_res_sinc} the energy is highly concentrated
  in lower frequency bands, and the local maxima exhibit sharp peaks in the target function. Depending on these factors, the
  multiresolution scheme might lead to different efficiency gains in different images.

\begin{table}[h!]
  \centering
  \begin{tabular}{c|c|c}
    decimation & ruled out (rot.) & ruled out (rot. + trans.)
    \\
    \hline
    $m=32$ & 99.25 \%   &99.97 \% 
    \\
    $m=16$ & 0.500 \%   &0.0220 \%
    \\
    $m=8$ & 0.139 \%   &0.0055 \%
    \\
    $m=4$ & 0.028 \%   & 7.0e-6 \%
    \\
    $m=2$ & 0.056 \%  &  1.8e-6 \%
    \\
    $m=1$ &  0 \%&    0 \%
  \end{tabular}
  \caption{The percentage of grid elements ruled out at each level of resolution for a search grid of rotation (column 2) and rotation plus translation (column 3) parameters.}
 \label{tab:computation_per_res}
\end{table}

\section{Embedding into a global optimization framework}
\label{sec:lipschitz_optim}
In order to demonstrate the practical value of our results, we present an example of how the multiresolution technique can be
integrated into the a globally optimal search algorithm. Here, we consider the Lipschitz optimization framework
\cite{Hansen1995_lipschitz}, a branch and bound approach which exploits the concept of Lipschitz continuity to obtain upper
bounds within each sub-region.  A function $h\colon D {\subseteq} \Real^n {\rightarrow} \Real$ is called Lipschitz continuous if there
exists a constant $L$ such that for any $\x_1,\x_2 \in D$
\begin{align}
  \label{eq:lipschitz_bound}
  |h(\x_2) - h(\x_1)| \leq L\,\norm{\x_2-\x_1}.
\end{align}
The smallest such $L$ is called the Lipschitz constant. For differentiable $h$ the Lipschitz constant
is the supremum of the gradient size within $D$. If $h(\x,\y)$ has Lipschitz constants $L_\x$ and $L_\y$ with respect to $\x$
and $\y$ respectively, then
\begin{align}
  \label{eq:lipschitz_bound_combine}
  \scalebox{.98}{$\left|h(\x_2,\y_2) - h(\x_1,\y_1)\right| \leq L_\x\norm{\x_2{-}\x_1} + L_\y\norm{\y_2{-}\y_1}$}.
\end{align}
In Lipschitz optimization, the target function is computed at the centre of each sub-domain, and then, using
\eqref{eq:lipschitz_bound} or \eqref{eq:lipschitz_bound_combine} an upper-bound is found on the target values
within the sub-domain. The sub-domain is either rejected or split based on whether or not the upper bound is smaller than the
currently best target value. 


\subsection{Lipschitz bounds}
\label{sec:lipschitz_bounds}
The efficiency of the algorithm highly depends on the quality of the Lipschitz bounds. An example of Lipschitz bounds for
intensity-based registration is given in \cite{Cremers2008_shape_priors} in the context of shape models. Here, we present better
bounds by looking at the target function in the frequency domain. According to our experiments, our bounds are usually smaller
by at least one order of magnitude. Let $\F(\z) = [\mathrm{real}(F(\z)), \mathrm{imag}(F(\z))]^T \in \Real^2$, and similarly
define $\G(\z)$. Then using the fact that $F(-\z) = \overline{F(\z)}$ and $G(-\z) = \overline{G(\z)}$ for real-valued $f$ and
$g$, we can rewrite \eqref{eq:target_freq} as
\begin{align}
\label{eq:target_freq_matrix}
Q(\mR,\t) = \int \F(\z)^T \, \mGamma(2\pi \t^T \z) \, \G(\mR \z) \, d\z.
\end{align}
where $\mGamma(\phi)$ is the 2D rotation matrix of angle $\phi$, that is 
\begin{align}
  \label{eq:rot_mat}
  \mGamma(\phi) = \Mat{cc}{\cos(\phi) & -\sin(\phi) \\ \sin(\phi) & \cos(\phi)}.
\end{align}
Now, the gradient of the target with respect to $\t$ is
\begin{align}
  \label{eq:d_dt_target_freq}
  \frac{d}{d \t} Q(\mR,\t) = 2 \pi \int \z \, \F(\z)^T \, \mGamma'(2\pi \t^T \z) \, \G(\mR \z) \, d\z.
\end{align}
where $\mGamma'(\phi) = \frac{d}{d\phi} \mGamma(\phi)$, which is also a rotation matrix. To obtain a fairly good upper bound on
the magnitude of \eqref{eq:d_dt_target_freq} we first divide the frequency domain into radial bands $\Omega_0, \Omega_1,
\ldots, \Omega_P$, where $\Omega_i = \{\z \,|\, r_i \leq \norm{\z} < r_{i+1}\}$ for radii $0=r_0 < r_1 < \cdots < r_{P+1}$. This is
similar to the approach taken in Sect. \ref{sec:radial_freq_bands}, only instead of partitioning the area outside $\Omega$, we
partition the entire frequency domain (look at Fig. \eqref{fig:nested} and replace $\Omega$ by $\Omega_0$).  Now, from \eqref{eq:d_dt_target_freq} we can write
\begin{align}
  \label{eq:lipschitz_t}
  \norm{\frac{d}{d \t} Q(\mR,\t)} & \leq 2 \pi \sum_{i=0}^P \int_{\Omega_i} \norm{\z} \, |\F(\z)^T \, \mGamma'(2\pi \t^T \z) \, \G(\mR \z)| \, d\z
  \nonumber\\
  &\hspace{-1.5cm} \leq 2 \pi \sum_{i=0}^P (\inf_{\Omega_{i}} \norm{\z}) \int_{\Omega_i}  \, |\F(\z)^T \, \mGamma'(2\pi \t^T \z) \, \G(\mR \z)| \, d\z
  \nonumber\\
  &\hspace{-1.5cm} \leq 2 \pi \sum_{i=0}^P r_{i+1} \int_{\Omega_i}  \, \norm{\F(\z)} \norm{\G(\mR \z)} \, d\z
  \nonumber\\
  &\hspace{-1.5cm} \leq 2 \pi \sum_{i=0}^P r_{i+1} \sqrt{\int_{\Omega_i} \norm{\F(\z)}^2 d\z} \sqrt{\int_{\Omega_i} \norm{\G(\z)}^2 d\z}.
\end{align}
In the third line of the above we used the fact that $\mGamma'(2\pi \t^T
\z)$ is a rotation matrix.  Also, notice that $\norm{\F(\z)}$ is simply equal to the magnitude of the complex quantity
$F(\z)$. The integrals can be computed similarly to Sect. \ref{sec:radial_freq_bands}. What makes the above a good bound is that
large values of $r_{i+1}$ are multiplied by high-frequency components of $F$ and $G$, which are typically small in natural
images. This has been possible due to the division of the frequency domain into radial bands $\Omega_i$. 

The parameterization of the rotation matrix can be quite different in 2D and 3D. Assume that $\rho {\in} \Real$ is a single parameter representing 
one degree of freedom in the rotation space. 
Then
\begin{align}
  \label{eq:d_rho_target_freq}
  \frac{d}{d \rho} Q(\mR,\t) =  \int \F(\z)^T \, \mGamma(2\pi \t^T \z) \, J_\G(\mR \z) \frac{d \mR}{d \rho} \z \, d\z,
\end{align}
where $J_\G$ is the Jacobian matrix of $\G$. It follows
\begin{align}
  |\frac{d}{d \rho} Q(\mR,\t)| & \leq \sum_{i=0}^P \int_{\Omega_i}  | \F(\z)^T \, \mGamma(2\pi \t^T \z) \, J_\G(\mR \z) \frac{d \mR}{d \rho} \z | \, d\z
  \nonumber\\
  &\hspace{-1.5cm} \leq \sum_{i=0}^P  \sqrt{\int_{\Omega_i} \!\!\norm{\F(\z)}^2 d\z} \sqrt{\int_{\Omega_i} \norm{J_\G(\z) \frac{d \mR}{d \rho} \mR^{-1}  \z}^2 d\z}.   \label{eq:lipschitz_rho}
\end{align}
For 2D the only parameter is the angle of rotation $\rho=\theta$. Notice that $\frac{d \mR}{d \theta}
\mR^{-1}$ is equivalent to a 90 degrees rotation counterclockwise. Thus, we have
\begin{align}
  \label{eq:d_dth_Ri_z}
  J_\G(\z) \frac{d \mR}{d \theta} \mR^{-1}  \z = J_\G(\z) \z^{\perp},
\end{align}
where $\z^{\perp} = [-z_2, z_1]^T$ is the vector $\z$ rotated by 90 degrees. 

For the more complex 3D case  we use the upper bound:
\begin{align}
  \label{eq:lipschitz_rho_bound2}
  \norm{J_\G(\z) \frac{d \mR}{d \rho} \mR^{-1}  \z} \leq \norm{J_\G(\z)} \norm{\frac{d \mR}{d \rho} \mR^{-1}  \z},
\end{align}
where $\norm{J_\G(\z)}$ is the spectral norm of the matrix $J_\G(\z)$. We use  the axis-angle
representation of the rotation, with the unit vector $\vOmega$ representing the axis of rotation and $\theta$ representing the rotation angle.  The rotation
can be formulated as
\begin{align}
  \label{eq:rotation_axis_angle}
  \mR\z = \cos\theta \, (\mI{-}\vOmega\vOmega^T)\,\z + \sin\theta \, (\vOmega {\times} \z) + \vOmega\vOmega^T\,\z
\end{align}
where $\times$ denotes the cross product. Similarly, $\mR^{-1}\z$ is obtained by replacing $\theta$ in the above with $-\theta$. 
Simple calculations give  
\begin{align}
  \label{eq:}
  \frac{d \mR}{d\theta} \mR^{-1}\z = \vOmega {\times} \z \leq \norm{\z}.
\end{align}
Now, we parameterize $\vOmega$ in the spherical coordinates using two angles
$\phi$ and $\psi$, as $\vOmega = [\cos\phi \cos\psi,~\sin\phi \cos\psi,~\sin\psi]^T$. Let $\vTau = [-\sin\phi ,~\cos\phi ,~0]^T$
and $\vNu = [-\cos\phi \sin\psi,~-\sin\phi \sin\psi,~\cos\psi]^T$, and notice that $\vOmega$, $\vTau$ and $\vNu$ form an
orthonormal basis with $\vOmega \times \vTau = \vNu$. Further, we have 
\begin{align}
  \frac{d\vOmega}{d\phi} = \cos\psi~\vTau, ~~~
  \frac{d\vOmega}{d\psi} = \vNu.
\end{align}
By taking $s = \sin(\frac{\theta}{2})$ and $c = \cos(\frac{\theta}{2})$, and using $\vNu {\times} \z = (\vTau\vOmega^T {-} \vOmega \vTau^T)\,\z$ we get
\begin{align}
  \label{eq:d_dphi_axis_angle}
  \frac{d \mR}{d\psi} \z &= 2 s^2 \, (\vOmega\vNu^T{+}\vNu\vOmega^T)\,\z + 2 c s\, (\vNu {\times} \z)   
  = 2 s \, \mU \mS  \mU^T \z,
\end{align}
where $\mU = [\vOmega, \vTau, \vNu]$ and 
\begin{align}
  \label{eq:U_S}
  \mS = \Mat{ccc}{0 & -c & s\\c  &0 & 0 \\ s & 0 &0}.
\end{align}
As $\mU$ is orthogonal and $\mS$ has spectral norm 1, we have 
\begin{align}
  \label{eq:d_dpsi_axis_angle_bound}
  \norm{\frac{d \mR}{d\psi} \mR^{-1}\z} &= \norm{2 s \, \mU \mS  \mU^T \mR^{-1}\z} \leq 2 \,|\sin({\theta}/{2})| \norm{\z}.
\end{align}
In a similar way, for $\phi$ one can obtain 
\begin{align}
  \label{eq:d_dphi_axis_angle_bound}
  |\frac{d \mR}{d\phi} \mR^{-1}\z| \leq 2 \, |\sin({\theta}/{2})| \, |\cos(\psi)| \, \norm{\z}. 
\end{align}
Lipschitz bounds may be obtained from \eqref{eq:d_dpsi_axis_angle_bound} and \eqref{eq:d_dphi_axis_angle_bound} by simply
maximizing $|\sin({\theta}/{2})|$ and $|\cos(\psi)|$ within each sub-region. By substituting into \eqref{eq:lipschitz_rho}, one 
realizes that the second square root linearly grows with the size of $\z$. However, this is compensated by the fact that
larger vectors $\z$ lie inside higher frequency bands $\Omega_i$, and thus, are multiplied by high-frequency components of $F$
(and $J_\G$) which are typically small.
Another important observation is that the Lipschitz bounds get smaller as images are more decimated. This can be easily seen by looking at
\eqref{eq:lipschitz_t} and \eqref{eq:lipschitz_rho}, and noticing that the spectrum of a low-resolution image is obtained by zeroing
out the  high-resolution spectrum  outside a certain ball.

\subsection{A multiresolution branch and bound algorithm}
\label{sec:multires_Lipschitz}
We divide the registration parameter space into a number of hypercubes. In Lipschitz optimization, each hypercube either gets
rejected or gets split depending on whether or not the corresponding upper bound on the target value is smaller than the largest
 target value found so far. The idea behind our multiresolution algorithm is to reduce the computations by using lower
resolution images, as much as possible, to make the reject/split decision. To see how this is done, consider a certain cube
$\tC$ whose centre corresponds to parameters $\mR_c,\t_c$. Like Sect. \ref{sec:basic_algorithm}, we  have a set of resolution levels $l=
0,1,\ldots,l_{\text{max}}$. For each resolution level $l$ there are two different bounds: the inter-resolution bound $B^l_{\text{res}}$ and
the Lipschitz bound $B^l_{\text{Lip}}$. The Lipschitz bound can be obtained from \eqref{eq:lipschitz_bound} or
\eqref{eq:lipschitz_bound_combine}, and bounds the variation of $Q^l(\mR,\t)$ from $Q^l(\mR_c,\t_c)$ within the cube $\tC$, that is
$|Q^l(\mR,\t) - Q^l(\mR_c,\t_c)| \leq B^l_{\text{Lip}}$ for all $(\mR,\t) \in \tC$. The inter-resolution bound gives $|Q^l(\mR_c,\t_c)
- Q(\mR_c,\t_c)| \leq B^l_{\text{res}}$. Thus, we can define a total bound $B^l_{\text{total}} {=}
B^l_{\text{res}} + B^l_{\text{Lip}}$ giving
\begin{align}
  \label{eq:total_bound}
  |Q(\mR,\t) - Q^l(\mR_c,\t_c)| \leq B^l_{\text{total}}~~~~\text{for all}~(\mR,\t) \in \tC.
\end{align}
Therefore, if $Q^l(\mR_c,\t_c) {+} B^l_{\text{total}}$ is smaller than $Q^*$, the currently best target value, we can safely
discard  $\tC$. Otherwise, we test the above for the next (higher) resolution level $l{-}1$. More precisely, our strategy is as
follows.  First, for any cube $\tC$ we define $l_{\text{min}}$ as $\arg\!\min_l B^l_{\text{total}}$. For each
cube, we start from  $l=l_{\text{max}}$ (the lowest resolution) and examine the upper bounds for $l=l_{\text{max}}, \ldots, l_{\text{min}}$ in order. 
If at any resolution $l$ it happened that $Q^l(\mR_c,\t_c) +
B^l_{\text{total}} \leq Q^*$, then we discard  $\tC$. Otherwise, we split the cube. We do not go further below
 $l_{\text{min}}$ as otherwise $B^l_{\text{total}}$ would become larger and there would be a high change of $Q^l(\mR_c,\t_c) +
B^l_{\text{total}}$ falling above $Q^*$ again.

Notice that, $B^l_{\text{res}}$ is equal or proportional to the residual of the energy after low-pass filtering. Thus, we have
$0 {=} B^0_{\text{res}} \leq B^1_{\text{res}} \leq \ldots \leq B^{l_{\text{max}}}_{\text{res}}$. On the contrary, we observed in
the previous section that $B^0_{\text{Lip}} \geq B^1_{\text{Lip}} \geq \ldots \geq B^{l_{\text{max}}}_{\text{Lip}}$. On the
first iterations of the algorithm, the cubes are large, and therefore, $B^l_{\text{Lip}}$-s dominate
$B^l_{\text{res}}$-s. Consequently, $l_{\text{min}}$ becomes equal or close to $l_{\text{max}}$, and only lower resolution
images are used for making the reject/split decision. As the algorithm goes on, the cubes get divided into smaller cubes. Thus,
$B^l_{\text{Lip}}$ becomes smaller compared to $B^l_{\text{res}}$, and $l_{\text{min}}$ moves towards zero. Therefore, higher
resolution images are examined. Due to this strategy, a large portion of the search space is explored in the lower-resolution
domains.

Algorithm~\ref{alg:lipschitz} describes our approach.  It uses the \emph{breadth-first search} strategy, using a \emph{queue} data
structure. We use a special element called \textbf{level-delimiter} to separate different levels of the search tree in the
queue. After searching all cubes of each level, we have a candidate optimal parameter $(\mR^{l*},\t^{l*})$, which is the
maximizer of $Q^{l_\text{min}}(\mR_c{,}\t_c){-}B^{l_\text{min}}$ among all cubes in this level (see line
\ref{alg:line:low-res-max} of the algorithm).  We examine the original  cost function $Q {=} Q^0$
on this candidate parameter set. If $Q(\mR^{l*},\t^{l*}) > Q^*$, then $Q^*$, $\mR^{*}$ and $\t^{*}$ are updated (line
\ref{alg:line:update-level}).


It is not always necessary to split across every dimension of the cube, when a split is needed. Notice that the bound $B^l_{\text{Lip}}$ is the sum of
the bounds for each parameter, as shown in \eqref{eq:lipschitz_bound_combine} for two parameters. We sort the bounds descendingly,
start from the parameter with the biggest bound, and keep splitting across the parameters, until the sum of the rest of the bounds
are smaller than $Q^* - B^{l_{\text{min}}}_{\text{res}}$, which is the margin required for rejecting the cube at resolution
$l_{\text{min}}$.

 The algorithm continues until $\Call{FinishCondition}{\,}$ is satisfied. Different criteria can be used as the finishing
 condition, such as when all the remaining cubes are small enough, or if the target function does not change much between two
 consecutive search levels. However, what is usually used in the branch and bound algorithms is to ensure that the currently best
 target $Q^*$ is within a certain distance $\varepsilon$ of the optimal target. For this, first, for each certain cube, we  compute an
 upper bound on the target function as the minimum of $Q^l(\mR_c,\t_c) +B^l_{\text{total}}$ over $l{=}l_{\text{min}}, \ldots,
 l_{\text{max}}$. The upper-bound $Q^{\text{up}}$ is the maximum of this quantity among all cubes in the queue. The algorithm
 finishes when $|Q^{\text{up}} - Q^*| < \varepsilon$. 


\begin{algorithm}
  \caption{A multiresolution Lipschitz optimization algorithm for rigid intensity-based registration. }\label{alg:lipschitz}
  
  \begin{algorithmic}[1]
    \Require 
    \Statex $\{f_\i^l\},\{g_\i^l\}$: registration image pairs for each resolution $l$, 
    \Statex $B^l_{\text{res}}$: inter-resolution bounds for each $l$,
    \Statex $L^l$: Lipschitz constant estimates for all parameters, for each $l$,
    \Statex $\tC_0$: the initial hypercube containing the parameter space, 
    \Statex $\mR_0,\t_0$: an initial solution (optional).
    \State $Q^* \gets Q(\mR_0,\t_0)$ or $-\infty$   \Comment{so-far best target value}
    \State $Q^{l*} \gets -\infty$   \Comment{best low-resolution lower-bound on target}
    
    
    \State \textsc{Queue}.\Call{Enqueue}{$\tC_0$}
    \State \textsc{Queue}.\Call{Enqueue}{\textbf{level-delimeter}}

    \Repeat  
    \State $\tC \gets $ \textsc{Queue}.\Call{Dequeue}{\,}  
    \If{$\tC =$ \textbf{level-delimeter}} \Comment{level completed}
    \State \scalebox{.97}{$Q^*,\mR^{*},\t^{*} {\gets} \Call{Amax}{Q^*, \mR^{*},\t^{*}, Q(\mR^{l*},\t^{l*}), \mR^{l*},\t^{l*}}$} \label{alg:line:update-level}
    \State \textsc{Queue}.\Call{Enqueue}{\textbf{level-delimeter}}
    \State \textbf{go to line} \ref{alg:line:until} \Comment{next iteration}
    \EndIf
    
    \State $\mR_c,\t_c \gets \Call{GetCentre}{\tQ}$
    \State $B^l_{\text{Lip}} \gets \Call{LipschitzBound}{L^l,\tC}$ \textbf{for all} $l$
    \State $B^l_{\text{total}} \gets B^l_{\text{res}} + B^l_{\text{Lip}}$ \textbf{for all} $l$
    \State $l_{\text{min}} \gets \arg\!\min_l B^l_{\text{total}}$
    
    \For{$l \gets l_{\text{max}}$ \textbf{downto} $l_{\text{min}}$}  
    \If{ $Q^l(\mR_c,\t_c) + B^l_{\text{total}} \leq Q^*$}
    \State \textbf{go to line} \ref{alg:line:until} \Comment{reject $\tC$, next iteration}
    \EndIf

    \EndFor
    
    \State \scalebox{.915}{$Q^{l*}{,}\mR^{l*}{,}\t^{l*} {\gets} \Call{Amax}{Q^{l*}{,} \mR^{l*}{,}\t^{l*}, Q^{l_\text{min}}(\mR_c{,}\t_c){-}B^{l_\text{min}}_{\text{res}}{,} \mR_c{,}\t_c}$} \label{alg:line:low-res-max}

    \State $Q^*, \mR^{*},\t^{*} {\gets} \Call{Amax}{Q^*, \mR^{*},\t^{*}, Q^{l*}, \mR^{l*},\t^{l*}}$

    \For{$\tC'$ \textbf{in} \Call{Split}{$\tC$}}  
    
    \State \textsc{Queue}.\Call{Enqueue}{$\tC'$}
    
    \EndFor
    
    \Until{\Call{FinishCondition}{\,}} \label{alg:line:until}
    
    \State \Return $\mR^*,\t^*$
    
    \Statex

    \Procedure{Amax}{$Q_1,\mR_1,\t_1, Q_2,\mR_2,\t_2$}
    \State \textbf{if} $Q_1 \geq Q_2$ \textbf{then} \Return $Q_1,\mR_1,\t_1$ 
    \State ~~~~~~~~~~~~~~~\textbf{else\,} \Return $Q_2,\mR_2,\t_2$ 
    
    \EndProcedure

  \end{algorithmic}
\end{algorithm}

\section{Experiments}
\label{sec:experiments}
The only globally optimal approach we found on rigid intensity-based registration is \cite{Cremers2008_shape_priors}, running a
single-resolution Lipschitz optimization for 2D images with larger Lipschitz bounds than ours. Their proposed estimation of
Lipschitz constant is more than 20 times larger than ours for rotation and more than 10 times larger for translation for the
image pair of Fig. \ref{fig:lh_res_sinc}. Therefore, here we compare the single-resolution and multiresolution versions of our
own algorithm and report the speed gain offered by the multiresolution technique.

\subsection{Implementation}
\label{sec:technical}

\paragraph*{Algorithm setup} We use the low-resolution to high-resolution registration of Sect. \ref{sec:low_high_reg} with
discretized integration and the bound \eqref{eq:bound_supp_interp_one_sinc_hl_upsample}, where $g$ is upsampled by a factor of
2, and bilinear (or trilinear) interpolation is used.  The highest decimation rate is set to the largest power of 2 where
$E^h_{fg} {=} \norm{f^h} \norm{g^h}$ in \eqref{eq:LH_error} is less than 0.05 of $\norm{f} \norm{g}$. For the example of
Fig. \ref{fig:lh_res_sinc} this gives a maximum decimation rate of 32. However, for images with a less compact spectrum this can
be smaller. We have found experimentally that this is a safe threshold for avoiding errors caused by discretized integration. We
set the convergence threshold $\varepsilon$ to $0.01 \norm{f} \norm{g}$ (see Sect. \ref{sec:multires_Lipschitz}). For the starfish
example of Fig. \ref{fig:lh_res_sinc} this provides a solution with an accuracy of about 0.2 pixels for translation and 0.04
degrees for rotation.

\paragraph*{Platform} The main algorithm is implemented under C++, while constructing the multiresolution pyramid and the bounds
are done under Matlab. The machine has Intel Core i7-5500U 2.4GHz CPU and 8GB of RAM.

\paragraph*{Low-pass filtering and downsampling} To implement the ideal filter \eqref{eq:low_pass} we zero out the DFT
coefficients of the discrete image outside the corresponding ball around the origin, and then take inverse DFT. This is not in general equivalent to
applying the ideal filter \eqref{eq:low_pass} to the corresponding continuous image. But, we can get closer to the ideal case by
using higher-resolution DFTs.  Another potential source of error is due to downsampling.  After
taking inverse DFT, we downsample the resulting discrete image by starting at the top-left pixel $(1,1)$ and then sampling every
$m$-th pixel. This means that the corresponding continuous image is sampled only within the boundary of the discrete
image. However, there is no guarantee that the continuous image will be zero at the sampling locations outside the image
boundary after low-pass filtering. One may solve this by sampling outside the boundaries up to a certain margin.  Both these
sources of error are negligible for images with a fairly large dark (near zero) margin. We can obtain such an image by zero
padding, or reversing the intensity values for images with a bright background.

\paragraph*{Upsampling} To upsample, zero-valued higher-frequency components are added to the DFT of the discrete image, and
inverse DFT is taken. Again, this is not equal to upsampling with sinc interpolation, but is accurate enough for images
with a dark margin, or when higher-resolution DFT is used.


\paragraph*{Image intensity} All images are converted to grayscale prior to registration. For images with a white or light
background  the intensity values are inverted.

\subsection{2D registration}
First we try the starfish image image of Fig. \ref{fig:lh_res_sinc}(a). The main image is $782{\times} 782$. The second image is
obtained by applying a translation of $(-50,20)$ pixels and a rotation of 45 degrees to the first image. 
 The multiresolution Lipschitz
optimization algorithm takes $44$ seconds to finish, while the single resolution version takes 3 hours and 13 minutes. This
means a large speedup of order 263.

Now, we apply our algorithm to register two different slices of the same brain MRI volume being 10 slices away from each
other (Fig. \ref{fig:reg_2D_brain}). We rotate the second slice by 160 degrees and and translate it by (-10,14) pixels. The
images are $192\by 156$. The single-resolution algorithm takes 534 seconds to perform the registration. The multiresolution
algorithm does it in 119 seconds. This time we get a 4.5 times speedup. This shows that the efficiency
gain can vary based on how concentrated the image is in the frequency domain and how the target value of the true solution stands
out in the registration parameter space.

\begin{figure}[h!]
  \centering
  \begin{tabular}{ccc}
    \includegraphics[width=0.10\textwidth]{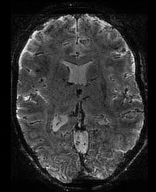}
    &
    \includegraphics[width=0.10\textwidth]{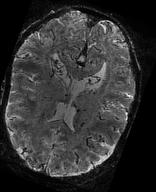}
    &
    \includegraphics[width=0.10\textwidth]{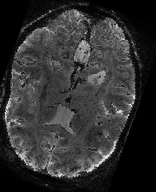}
    \\
    (a) & (b) & (c)
  \end{tabular}
  \caption{Registering two different slices of a brain MRI image. (a) the first slice, (b) the second slice rotated and
    translated, and (c) the first slice registered to the second slice.  The data is obtained from \cite{Forstmann2014_7Tesla}.}
  \label{fig:reg_2D_brain}
\end{figure}

\subsection{3D registration}
\label{sec:3D_reg}
Our current branch and bound algorithm is not yet able to handle the 6 degrees of freedom of a full 3D registration in a
reasonable time. What is done here is to align the centres of masses of the two images  and then perform a
rotation-only search around the centre of mass. We consider the binary volumes of the same vertebra in two different subjects
(Fig. \ref{fig:reg_3D_vertebra}). Both volumes are $81 \by 90 \by 24$ and are obtained by manually labeling CT images\footnote{The data 
can be obtained from CSI 2014 Workshop http://csi-workshop.weebly.com/challenges.html, also see \cite{Yao2012_vertebra}}. The accurate alignment of such images is usually required as the first step of constructing a shape
model.  The registration results can be seen in Fig. \ref{fig:reg_3D_vertebra}.  The single-resolution algorithm performs the
registration in 19 minutes and 47 seconds.  For the multiresolution approach it takes 29.8 seconds, which is a 40-fold
speedup.

\begin{figure}[h!]
  \centering
  \begin{tabular}{ccc}
    \includegraphics[width=0.11\textwidth]{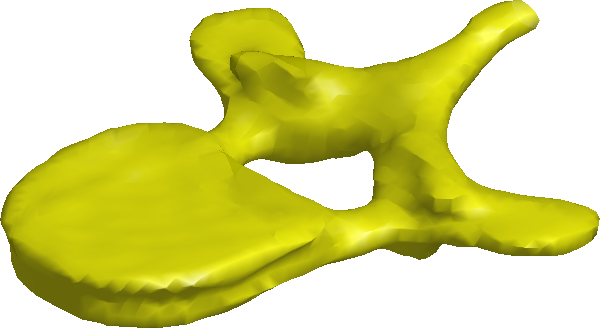}
    &
    \includegraphics[width=0.11\textwidth]{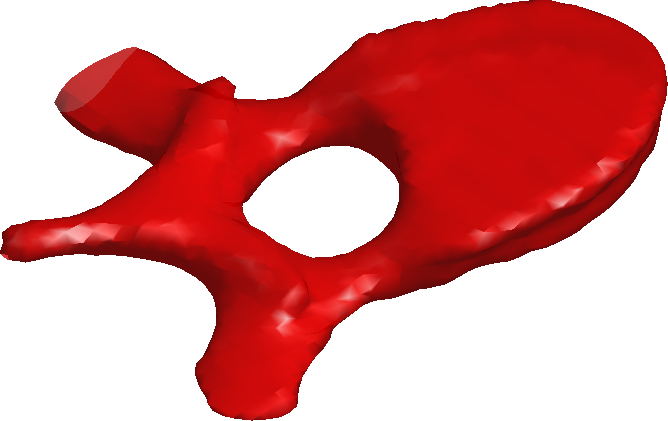}
    &
    \includegraphics[width=0.11\textwidth]{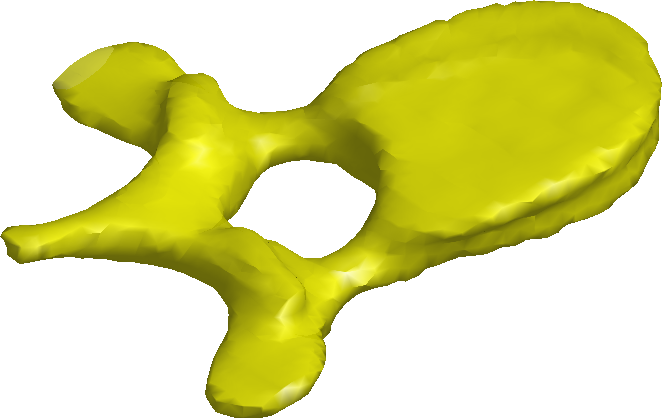}
    \\
    (a) & (b) & (c)
  \end{tabular}
  \caption{3D registration of binary vertebra shapes, (a,b) the registration image pair, (c) the first image registered to
    the second image.}
  \label{fig:reg_3D_vertebra}
\end{figure}

\subsection{Reflective symmetry detection}
A closely related problem to rigid registration is detecting the axis or plane of symmetry in 2D and 3D images. Most natural images are not perfectly
symmetric. One, however, can estimate the best choice by maximizing the correlation of an image with its mirror reflection across a
line (2D) or a plane (3D). The problem formulation and other details are given in Appendix \ref{sec:PoS}. Here, we demonstrate the application of 
the multiresolution Lipschitz algorithm.

Fig. \ref{fig:AoS} shows three examples of symmetry detection in 2D images, comparing the execution time of single-resolution
and multiresolution algorithms. In all cases the multiresolution algorithm is faster by a factor of more than 4.  The reader 
may have noticed that the \emph{butterfly} image takes more time than the \emph{face} image despite being smaller. This is due to sharper
edges (larger high-frequency components) in the butterfly image, leading to larger Lipschitz bounds.

\begin{figure}[h!]
   \centering
   \begin{tabular}{c|c|c|c}
     &
     \includegraphics[width=0.1\textwidth]{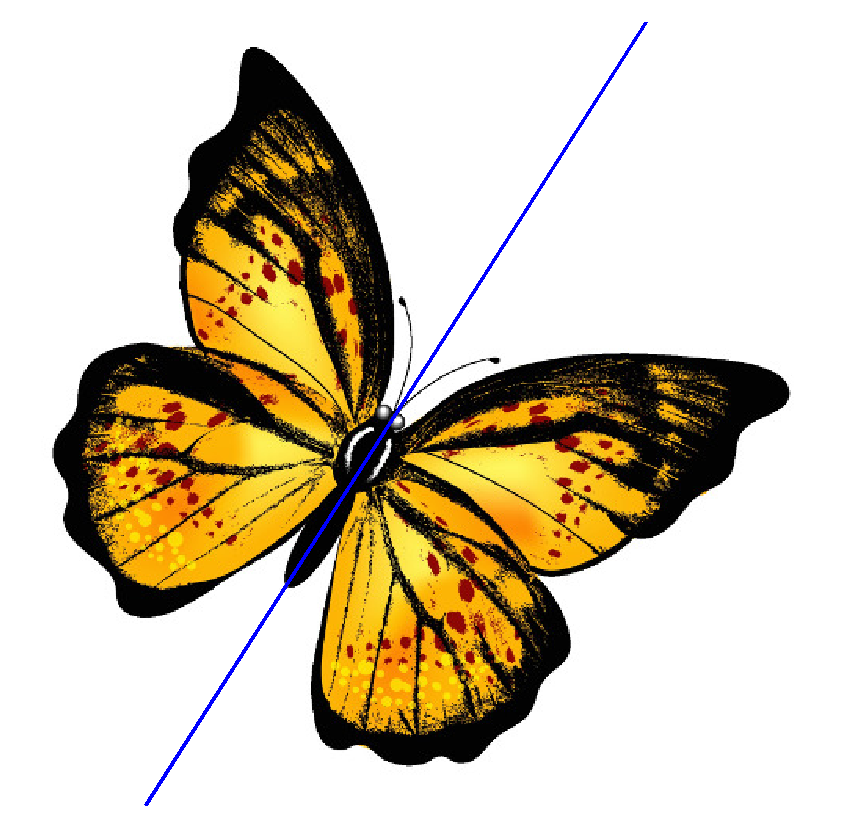}
     & 
     \includegraphics[width=0.09\textwidth]{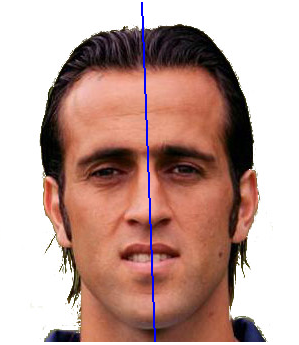}
     & 
     \includegraphics[width=0.08\textwidth]{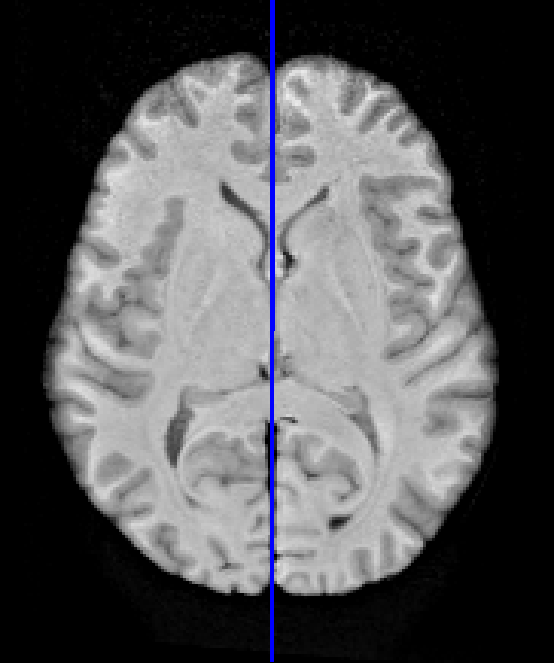}
     \\
    \hline
    resolution & $300\by 300$ & $340\by 320$ &  $218\by 182$ 
    \\
    \hline
    multi-res. time  & 14.7s & 3.9s & 0.92s
    \\
    \hline
    single-res. time & 64.3s & 16.5s & 4.4s
    \\
    \hline
    speedup            & 4.4 & 4.2 & 4.7
  \end{tabular}
  \caption{Detecting axis of symmetry for images of a butterfly, a face and a brain MRI slice. The butterfly image is from
    http://hdwallpapersrocks.com. The brain MRI image was obtained from http://www.brain-map.org.}
  \label{fig:AoS}
\end{figure}


For estimating plane of symmetry, we use the binary volume of a human vertebra (Fig. \ref{fig:PoS}). We use the same data 
as for 3D registration. Reflective symmetry detection for binary volumes is specially useful in shape analysis
and modeling. The 3D image is $114 \by 106 \by 44$. The single-resolution algorithm takes 2 hours and 51 minutes to detect the
optimal plane of symmetry. The multiresolution algorithm does this in 29.2 seconds. This is about 350 times faster, which is a
considerable speedup.


\begin{figure}[h!]
  \centering
  \includegraphics[width=0.20\textwidth]{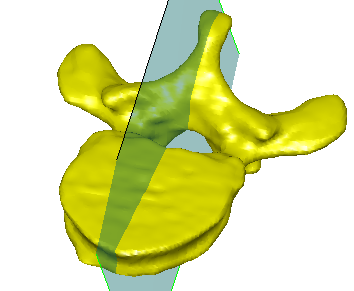}
  \caption{Estimating the plane of symmetry for the binary volume of a vertebra
  }
  \label{fig:PoS}
\end{figure}

\section{Conclusion and future work}
We showed that low-resolution target values can tightly bound the high-resolution target function for rigid intensity based
image registration. This led to a multiresolution search scheme in which the search at each resolution limits the search space
for the next higher resolution level. By embedding into the Lipschitz optimization framework, we showed that this strategy can
significantly speed up the globally optimal registration algorithms. This paper was mostly focused on providing a working
example of such an embedding and demonstrating the efficiency gains it provides. Optimizing the efficiency of the
single-resolution algorithm itself can be the subject of future research. This can be done by finding better Lipschitz
bounds or employing more effective search strategies instead of the breadth first method. An extension of this work could be to
find bounds for more general transformations (similarity, affine, nonlinear, etc.) and other types of target functions (robust,
information theoretic, etc.). The effectiveness of our multiresolution approach comes from the energy compaction of natural
images in the Fourier basis. An important extension is to examine the application of similar ideas for other bases, and
obtaining optimized basis functions tailored for each image pair.

\appendices

\section{Bounded support interpolation}
\label{sec:bounded_support_interp_app}
\begin{proof}[Proof of Proposition \ref{thm:bound_supp_interp}]
The difference $\left|Q'(\mR,\t) - Q'^l(\mR,\t) \right|$ can be written in the frequency domain using  \eqref{eq:target_freq}:
\begin{align}
  \label{eq:LH_diff_interp}
  \resizebox{.88\columnwidth}{!}{$\left|\int \overline{F'(\z)} \, G'(\mR \z) \, e^{2\pi i \t^T \z}\, d\x - \int \overline{F'^{l}(\z)} \, G'^{l}(\mR \z) \, e^{2\pi i \t^T \z}d\x \right|$}.
\end{align}
Using the triangle inequality, we break the above 
into integrations over $\Omega$ and it complement $\bar{\Omega}$:
 \begin{align}
   \label{eq:Dz_separate} 
   \left|\int D(\z) d\x \right|   \leq  \left|\int_\Omega D(\z) d\x \right| + \left|\int_{\bar{\Omega}} D(\z) d\x \right|, 
\end{align}
 where $D(\z) = e^{2\pi i \t^T \z} \left(\overline{F'(\z)} \, G'(\mR \z) - \overline{F'^{l}(\z)} \, G'^{l}(\mR \z)\right)$. 
 Let us first deal with the integral inside $\Omega$. Since $F(\z) {=} F^l(\z)$ for all $\z \in \Omega$, we expect  
$F'(\x)$ and $F'^l(\x)$ to be close. A similar argument goes for $G'(\x)$ and $G'^l(\x)$. 
Now using the equality 
\begin{align}
  ab-a^lb^l = a(b{-}b^l) + b(a{-}a^l) - (a{-}a^l)(b{-}b^l)
\end{align}
for any four complex numbers $a,b,a^l$ and $b^l$ we can write
\begin{align}
  &\left|\int_\Omega \left(\overline{F'(\z)} \, G'(\mR \z ) - F'^{l}(\z) \, G'^l(\mR \z ) \right) e^{2\pi i \t^T \z}\right| 
  \nonumber\\
  \leq &\hspace{10pt}\left|\int_\Omega\overline{F'(\z)} \, \left(G'(\mR \z ) - G'^l(\mR \z )\right) e^{2\pi i \t^T \z} \right|
  \nonumber\\
  &+ \left|\int_\Omega G'(\mR \z ) \left(\overline{F'(\z)} - \overline{F'^l(\z)} \right) e^{2\pi i \t^T \z} \right|
  \nonumber\\
  &+  \left|\int_\Omega\left(\overline{F'(\z)} - \overline{F'^l(\z)} \right) \left(G'(\mR \z ) - G'^l(\mR \z )\right) e^{2\pi i \t^T \z} \right|
  \nonumber\\
  \leq &  \resizebox{.76\columnwidth}{!}{$\sqrt{\int_\Omega |F'(\z)|^2 d\z} \sqrt{\int_\Omega  \left|G'( \z ) - G'^l( \z )\right|^2 d\z}$}
  \nonumber\\
  &+ \resizebox{.76\columnwidth}{!}{$\sqrt{\int_\Omega |G'(\z)|^2 d\z} \sqrt{\int_\Omega  \left|F'( \z ) - F'^l( \z )\right|^2 d\z}$}
  \nonumber\\
  &+ \resizebox{.81\columnwidth}{!}{$\sqrt{\int_\Omega  \left|F'( \z ) {-} F'^l( \z )\right|^2 d\z} \sqrt{\int_\Omega  \left|G'( \z ) {-} G'^l( \z )\right|^2 d\z}$}
\end{align}
which is exactly equal to lines \eqref{eq:bound_supp_interp_in1} and \eqref{eq:bound_supp_interp_in2} of the bound. 
The first inequality above is due to the triangle inequality. The second inequality is obtained by Cauchy-Schwarz followed by a
change of variables $\z \leftarrow \mR\z$. 
Now, we consider the integral outside $\Omega$. Since, outside $\Omega$ the signals $F'$ and $F'^l$ (aslo $G'$ and $G'^l$) 
are not similar we consider them separately:
 \begin{align}
   &\left|\int_{\bar{\Omega}} \left(\overline{F'(\z)} \, G'(\mR \z ) - F'^{l}(\z) \, G'^l(\mR \z ) \right) e^{2\pi i \t^T \z}d\z\right| 
   \nonumber\\
   &\leq   \scalebox{.98}{$\left|\int_{\bar{\Omega}} \overline{F'(\z)} \, G'(\mR \z ) e^{2\pi i \t^T \z} d\z\right| 
   {+} 
   \left|\int_{\bar{\Omega}}  F'^{l}(\z) \, G'^l(\mR \z ) e^{2\pi i \t^T \z} d\z\right| \label{eq:HL_interp_diff_B_bar}$}
   \nonumber \\
   &\leq  \resizebox{.60\columnwidth}{!}{$\sqrt{\int_{\bar{\Omega}} |F'(\z)|^2 d\z}  \sqrt{\int_{\bar{\Omega}} |G'(\z)|^2 d\z} $}
   \nonumber\\
   &\hspace{10pt}+ 
    \resizebox{.60\columnwidth}{!}{$\sqrt{\int_{\bar{\Omega}} |F'^l(\z)|^2 d\z}  \sqrt{\int_{\bar{\Omega}} |G'^l(\z)|^2 d\z}$}.
 \end{align}
 This is equal to \eqref{eq:bound_supp_interp_out}.
\end{proof}

\section{}
\label{sec:bounded_support_interp_app_Fsinc}
To obtain \eqref{eq:bound_supp_interp_one_sinc2},  we treat the frequency areas $\Omega$ and
$\bar{\Omega}$ separately, like in \eqref{eq:Dz_separate}.  We use the same approach from which \eqref{eq:bound_supp_interp_in1} and
\eqref{eq:bound_supp_interp_in2} are obtained. However, since in this case $F'=F$ and $F'^l = F^l$, and also $F(\z) = F^l(\z)$
inside $\Omega$, then we have $E_{\Omega}(F' {-} F'^l) = E_{\Omega}(F {-} F^l) = 0$. Therefore, \eqref{eq:bound_supp_interp_in1} and
\eqref{eq:bound_supp_interp_in2}  reduce to $\sqrt{ E_{\Omega}(F) \, E_{\Omega}(G' {-} G'^l)}$ which is the first term in
\eqref{eq:bound_supp_interp_one_sinc2}.  Now, outside $\Omega$ we have $F^l(\z) = 0$, and therefore
\begin{align}
  \left|\int_{\bar{\Omega}} D(\z) d\x \right| 
  &=
  \left|\int_{\bar{\Omega}}  \overline{F(\z)} G'(\mR \z) e^{2\pi i \t^T \z} d\z \right|  
  \nonumber\\
  &=
  \left|\int_{\Omega'\setminus \Omega}  \overline{F(\z)} G'(\mR \z) e^{2\pi i \t^T \z} d\z \right|  \label{eq:FGp_l2}
  \\
  &\leq
  \resizebox{.62\columnwidth}{!}{$\sqrt{\int_{\Omega'\setminus \Omega} |F(\z)|^2 d\z}   \sqrt{\int_{\Omega'\setminus \Omega} |G'(\z)|^2 d\z} $} \label{eq:FGp_l3}
  \\
  & =  \sqrt{E_{\bar{\Omega}}^{\phantom{l}}(F)} \, \sqrt{E_{\Omega' \setminus \Omega}(G')}, \label{eq:FGp_l4}
\end{align}
where $D(\z) {=} \left(\overline{F(\z)} \, G'(\mR \z ) - F^{l}(\z) \, G'^l(\mR \z ) \right) e^{2\pi i \t^T \z}$ and $\Omega'$ is
the ball of radius $\frac{\sqrt{d}}{2T}$. This is equal to the second term in \eqref{eq:bound_supp_interp_one_sinc2}. The
equality \eqref{eq:FGp_l2} holds since  $F(\z)$ is bandlimited to $\frac{1}{2T}$ in every dimension, and hence the
integrand is zero outside $\Omega'$. The inequality \eqref{eq:FGp_l3} is obtained by Cauchy-Schwarz followed by a change of
variables $\z \leftarrow \mR \z$ in the second integral. The change of variables is possible since $\mR \z \in \Omega'
{\setminus} \Omega$ if and only if $\z \in \Omega' {\setminus} \Omega$.  Last line uses the fact that $E_{\Omega' \setminus
  \Omega}^{\phantom{l}}(F) = E_{\bar{\Omega}}^{\phantom{l}}(F)$.

\section{Axis and plane of symmetry}
\label{sec:PoS}
We represent a plane by its normal $\u$ and its distance $\alpha$ from the origin. 
The reflection of a point $\x$ with respect to the plane can be formulated as 
$\x'= (\mI{-}2\u\u^T) (\x - 2 \alpha \u)$. The correlation-based target function
then becomes 
\begin{align}
\label{eq:target_PoS}
Q(\u,\alpha) = \int f(\x) \, f\left((\mI{-}2\u\u^T) (\x - 2 \alpha \u)\right)\, d\x. 
\end{align}
All the inter-resolution bounds in the registration problem works here by simply replacing $g$ with $f$. One can verify this by
using $f$, $(\mI{-}2\u\u^T)$ and $-2 \alpha \u$, respectively, instead of $g$, $\mR$ and $\t$ in the derivations. It is only
left to find Lipschitz bounds for the new parameters $\alpha$ and $\u$. Following a similar approach as Sect. \ref{sec:lipschitz_bounds},
we can write $Q(\u,\alpha)$ as
\begin{align}
\label{eq:target_PoS_freq_matrix}
Q(\u,\alpha) = \int \F(\z)^T \, \mGamma(-4\pi \alpha\u^T \z) \, \F((\mI{-}2\u\u^T) \z) \, d\z.
\end{align}
To be concise, let $\beta = -4\pi \alpha\u^T \z$. A bound for $\alpha$ can be
\begin{align}
\label{eq:d_dalpha_target_PoS_freq_matrix_bound}
|\frac{d}{d\alpha} Q(\u,\alpha)| &= 4 \pi \left|\int \u^T\z \,\F(\z)^T \, \scalebox{1.14}{$\mGamma$}'(\beta) \, \F((\mI{-}2\u\u^T) \z) \, d\z\right|
\nonumber\\
&\leq 4 \pi \sum_{i=0}^P \int_{\Omega_i} \norm{\z} \norm{\F(\z)} \norm{\F((\mI{-}2\u\u^T) \z)} \, d\z
\nonumber\\
&\leq 4 \pi \sum_{i=0}^P r_{i+1} \int_{\Omega_i} \norm{\F(\z)}^2 d\z.
\end{align}
To bound the derivative with respect to $\u$ we represent it in polar/spherical 
coordinates. Assume $\gamma$ is a parameter by which $\u$ is parameterised, and 
let $\u' = \frac{d}{d\gamma} \u$. 
Then we have 
\begin{align}  
\label{eq:d_dgamma_target_PoS_freq_matrix}
\frac{d}{d\gamma} Q(\u,\alpha) = &{-}4 \pi \alpha\int {\u'}^T\!\z \,\F(\z)^T \, \scalebox{1.14}{$\mGamma$}'(\beta) \, \F((\mI{-}2\u\u^T) \z) \, d\z
\nonumber\\
&\hspace{-1.7cm} {-} 2 \!\!\int  \F(\z)^T  \scalebox{1.14}{$\mGamma$}(\beta)  J_\F((\mI{-}2\u\u^T) \z) (\u'\u^T{+}\u{\u'}^T) \z d\z.
\end{align}
As $\u'$ is always perpendicular to $\u$ (see polar/spherical parameterization of $\vOmega$ in Sect.~\ref{sec:lipschitz_bounds}), 
we have
\begin{align}
\label{eq:d_dgamma_target_PoS_freq_matrix_bound}
|\frac{d}{d\gamma} Q(\u,\alpha)| \leq& 4 \pi \alpha \norm{\u'}\sum_{i=0}^P r_{i+1} \, E_{\Omega_i}(F)
\nonumber\\
&\hspace{-2cm} {+} 2 \sum_{i=0}^P  \sqrt{E_{\Omega_i}(F) \int_{\Omega_i} \norm{J_\F(\z) \, (\u{\u'}^T{-}\u'\u^T)\,\z}^2  d\z}.
\end{align}
where $E_{\Omega_i}(F) = \int_{\Omega_i} \norm{\F(\z)}^2 d\z$.
For the 2D case $\u$ equals $[\cos(\phi), \sin(\phi)]^T$. Then 
\begin{align}
\label{eq:d_dphi_target_PoS_freq_matrix_bound}
|\frac{d}{d\phi} Q(\u,\alpha)| =& 4 \pi \alpha \sum_{i=0}^P r_{i+1} \,E_{\Omega_i}(F)
\nonumber\\
& \hspace{-2cm}+ 2 \sum_{i=0}^P E_{\Omega_i}(F)^{1{\!/}2} \, \sqrt{\int_{\Omega_i} \norm{J_\F(\z) \,\z^\perp}^2  \, d\z},
\end{align}
where $\z^\perp$ was defined in \eqref{eq:d_dth_Ri_z}. 
For 3D, we use spherical coordinates:  $\u = [\cos(\phi) \cos(\psi), \sin(\phi) \cos(\psi), \sin(\psi)]^T$. Thus,
\begin{align}
\label{eq:d_dphi_3D_target_PoS_freq_matrix_bound}
|\frac{d}{d\phi} Q(\u,\alpha)| \leq& 4 \pi \alpha\, |\cos\psi| \sum_{i=0}^P r_{i+1} \,E_{\Omega_i}(F)
\nonumber\\
& \hspace{-2cm}+ 2 \,|\cos\psi| \, \sum_{i=0}^P E_{\Omega_i}(F)^{1{\!/}2} \, \sqrt{\int_{\Omega_i} \norm{J_\F(\z)}^2 \norm{\z}^2  \, d\z},
\end{align}
where $\norm{J_\F(\z)}$ denotes the spectral norm of $J_\F(\z)$, and
\begin{align}
\label{eq:d_dpsi_3D_target_PoS_freq_matrix_bound}
|\frac{d}{d\psi} Q(\u,\alpha)| \leq& 4 \pi \alpha \sum_{i=0}^P r_{i+1} \,E_{\Omega_i}(F)
\nonumber\\
& \hspace{-2cm}+ 2 \sum_{i=0}^P \,E_{\Omega_i}(F)^{1{\!/}2} \, \sqrt{\int_{\Omega_i} \norm{J_\F(\z)}^2 \norm{\z}^2  \, d\z}.
\end{align}

\bibliographystyle{IEEEtran}
\bibliography{phd.bib}

\begin{thebibliography}{10}
\providecommand{\url}[1]{#1}
\csname url@samestyle\endcsname
\providecommand{\newblock}{\relax}
\providecommand{\bibinfo}[2]{#2}
\providecommand{\BIBentrySTDinterwordspacing}{\spaceskip=0pt\relax}
\providecommand{\BIBentryALTinterwordstretchfactor}{4}
\providecommand{\BIBentryALTinterwordspacing}{\spaceskip=\fontdimen2\font plus
\BIBentryALTinterwordstretchfactor\fontdimen3\font minus
  \fontdimen4\font\relax}
\providecommand{\BIBforeignlanguage}[2]{{%
\expandafter\ifx\csname l@#1\endcsname\relax
\typeout{** WARNING: IEEEtran.bst: No hyphenation pattern has been}%
\typeout{** loaded for the language `#1'. Using the pattern for}%
\typeout{** the default language instead.}%
\else
\language=\csname l@#1\endcsname
\fi
#2}}
\providecommand{\BIBdecl}{\relax}
\BIBdecl

\bibitem{Huttenlocher1992_Hausdroff}
D.~P. Huttenlocher and W.~J. Rucklidge, ``A multi-resolution technique for
  comparing images using the hausdorff distance,'' Cornell University, Tech.
  Rep., 1992.

\bibitem{Li2007_3D3Dreg}
H.~Li and R.~Hartley, ``The 3d-3d registration problem revisited,'' in
  \emph{Computer Vision, 2007. ICCV 2007. IEEE 11th International Conference
  on}, Oct 2007, pp. 1--8.

\bibitem{Yang2013_goicp}
J.~Yang, H.~Li, and Y.~Jia, ``Go-icp: Solving 3d registration efficiently and
  globally optimally,'' in \emph{The IEEE International Conference on Computer
  Vision (ICCV)}, December 2013.

\bibitem{Parra2014_3D_reg}
A.~J. Parra~Bustos, T.-J. Chin, and D.~Suter, ``Fast rotation search with
  stereographic projections for 3d registration,'' in \emph{Computer Vision and
  Pattern Recognition (CVPR), 2014 IEEE Conference on}.\hskip 1em plus 0.5em
  minus 0.4em\relax IEEE, 2014, pp. 3930--3937.

\bibitem{Pfeuffer2012_BB_medical}
F.~Pfeuffer, M.~Stiglmayr, and K.~Klamroth, ``Discrete and geometric branch and
  bound algorithms for medical image registration,'' \emph{Annals of Operations
  Research}, vol. 196, no.~1, pp. 737--765, 2012.

\bibitem{Olsson2009_bb_registration}
C.~Olsson, F.~Kahl, and M.~Oskarsson, ``Branch-and-bound methods for euclidean
  registration problems,'' \emph{Pattern Analysis and Machine Intelligence,
  IEEE Transactions on}, vol.~31, no.~5, pp. 783--794, 2009.

\bibitem{Hansen1995_lipschitz}
P.~Hansen and B.~Jaumard, ``Lipschitz optimization,'' in \emph{Handbook of
  Global Optimization}, R.~Horst and P.~M. Pardalos, Eds.\hskip 1em plus 0.5em
  minus 0.4em\relax Springer, 1995.

\bibitem{Cootes1995_ASM}
T.~F. Cootes, C.~J. Taylor, D.~H. Cooper, and J.~Graham, ``Active shape
  models-their training and application,'' \emph{Computer vision and image
  understanding}, vol.~61, no.~1, pp. 38--59, 1995.

\bibitem{Cremers2008_shape_priors}
D.~Cremers, F.~R. Schmidt, and F.~Barthel, ``Shape priors in variational image
  segmentation: Convexity, lipschitz continuity and globally optimal
  solutions,'' Anchorage, Alaska, Jun. 2008.

\bibitem{Corvi1995}
M.~Corvi and G.~Nicchiotti, ``Multiresolution image registration,'' in
  \emph{Image Processing, 1995. Proceedings., International Conference on},
  vol.~3, Oct 1995, pp. 224--227 vol.3.

\bibitem{Thevenaz2000}
P.~Thevenaz and M.~Unser, ``Optimization of mutual information for
  multiresolution image registration,'' \emph{Image Processing, IEEE
  Transactions on}, vol.~9, no.~12, pp. 2083--2099, Dec 2000.

\bibitem{Maes1999}
F.~Maes, D.~Vandermeulen, and P.~Suetens, ``Comparative evaluation of
  multiresolution optimization strategies for multimodality image registration
  by maximization of mutual information,'' \emph{Medical Image Analysis},
  vol.~3, no.~4, pp. 373 -- 386, 1999.

\bibitem{Li1995_successive}
W.~Li and E.~Salari, ``Successive elimination algorithm for motion
  estimation,'' \emph{Image Processing, IEEE Transactions on}, vol.~4, no.~1,
  pp. 105--107, 1995.

\bibitem{Lee1997_pyramid}
C.-H. Lee and L.-H. Chen, ``A fast motion estimation algorithm based on the
  block sum pyramid,'' \emph{Image Processing, IEEE Transactions on}, vol.~6,
  no.~11, pp. 1587--1591, 1997.

\bibitem{Gao2000_multilevel}
X.~Gao, C.~Duanmu, and C.~Zou, ``A multilevel successive elimination algorithm
  for block matching motion estimation,'' \emph{Image Processing, IEEE
  Transactions on}, vol.~9, no.~3, pp. 501--504, 2000.

\bibitem{Chen2001_winner_takes_all}
Y.-S. Chen, Y.-P. Hung, and C.-S. Fuh, ``Fast block matching algorithm based on
  the winner-update strategy,'' \emph{Image Processing, IEEE Transactions on},
  vol.~10, no.~8, pp. 1212--1222, 2001.

\bibitem{Di2003_bounded_partial_correlation}
L.~Di~Stefano and S.~Mattoccia, ``Fast template matching using bounded partial
  correlation,'' \emph{Machine Vision and Applications}, vol.~13, no.~4, pp.
  213--221, 2003.

\bibitem{Hel2005_proj_kernels}
Y.~Hel-Or and H.~Hel-Or, ``Real-time pattern matching using projection
  kernels,'' \emph{Pattern Analysis and Machine Intelligence, IEEE Transactions
  on}, vol.~27, no.~9, pp. 1430--1445, 2005.

\bibitem{Mattoccia2008_bounded_corr}
S.~Mattoccia, F.~Tombari, and L.~D. Stefano, ``Fast full-search equivalent
  template matching by enhanced bounded correlation,'' \emph{Image Processing,
  IEEE Transactions on}, vol.~17, no.~4, pp. 528--538, 2008.

\bibitem{Tombari2009_inc_dissimilarity}
F.~Tombari, S.~Mattoccia, and L.~Di~Stefano, ``Full-search-equivalent pattern
  matching with incremental dissimilarity approximations,'' \emph{Pattern
  Analysis and Machine Intelligence, IEEE Transactions on}, vol.~31, no.~1, pp.
  129--141, 2009.

\bibitem{Ouyang2009_Walsh_Hadamard}
W.~Ouyang and W.-K. Cham, ``Fast algorithm for walsh hadamard transform on
  sliding windows,'' \emph{IEEE Transactions on Pattern Analysis \& Machine
  Intelligence}, no.~1, pp. 165--171, 2009.

\bibitem{Ouyang2012_performance_eval}
W.~Ouyang, F.~Tombari, S.~Mattoccia, L.~D. Stefano, and W.-K. Cham,
  ``Performance evaluation of full search equivalent pattern matching
  algorithms,'' \emph{Pattern Analysis and Machine Intelligence, IEEE
  Transactions on}, vol.~34, no.~1, pp. 127--143, 2012.

\bibitem{Gharavi2001_low_res_prune}
M.~Gharavi-Alkhansari, ``A fast globally optimal algorithm for template
  matching using low-resolution pruning,'' \emph{Image Processing, IEEE
  Transactions on}, vol.~10, no.~4, pp. 526--533, 2001.

\bibitem{Korman2013_affine_template_matching}
S.~Korman, D.~Reichman, G.~Tsur, and S.~Avidan, ``Fast-match: Fast affine
  template matching,'' in \emph{Computer Vision and Pattern Recognition (CVPR),
  2013 IEEE Conference on}.\hskip 1em plus 0.5em minus 0.4em\relax IEEE, 2013,
  pp. 2331--2338.

\bibitem{Kazhdan2004_reflective_symmetry}
M.~Kazhdan, B.~Chazelle, D.~Dobkin, T.~Funkhouser, and S.~Rusinkiewicz, ``A
  reflective symmetry descriptor for 3d models,'' \emph{Algorithmica}, vol.~38,
  no.~1, pp. 201--225, 2004.

\bibitem{Forstmann2014_7Tesla}
B.~U. Forstmann, M.~C. Keuken, A.~Schafer, P.-L. Bazin, A.~Alkemade, and
  R.~Turner, ``Multi-modal ultra-high resolution structural 7-tesla mri data
  repository,'' \emph{Scientific data}, vol.~1, 2014.

\bibitem{Yao2012_vertebra}
J.~Yao, J.~E. Burns, H.~Munoz, and R.~M. Summers, ``Detection of vertebral body
  fractures based on cortical shell unwrapping,'' in \emph{Medical Image
  Computing and Computer-Assisted Intervention--MICCAI 2012}.\hskip 1em plus
  0.5em minus 0.4em\relax Springer, 2012, pp. 509--516.

\end{thebibliography}

\end{document}